\documentclass{article}

\usepackage{subcaption}

\usepackage[letterpaper, left=1in, right=1in, top=1in, bottom=1in]{geometry}

\PassOptionsToPackage{hypertexnames=false}{hyperref}  
\usepackage{parskip}

\usepackage[dvipsnames]{xcolor}
\usepackage[colorlinks,linkcolor=magenta,citecolor=blue, pagebackref=true,backref=true]{hyperref}
\renewcommand*{\backrefalt}[4]{%
    \ifcase #1 \footnotesize{(Not cited.)}%
    \or        \footnotesize{(Cited on page~#2.)}%
    \else      \footnotesize{(Cited on pages~#2.)}%
    \fi}

\usepackage{microtype}
\usepackage{hhline}
\usepackage{bm,bbm}

\makeatletter
\newcommand{\neutralize}[1]{\expandafter\let\csname c@#1\endcsname\count@}
\makeatother

\usepackage{algorithm}

\usepackage{natbib}
\bibpunct{(}{)}{;}{a}{,}{,}

\usepackage{amsthm}
\usepackage{mathtools}
\usepackage{amsmath}
\usepackage{bbm}
\usepackage{amsfonts}
\usepackage{amssymb}

\usepackage{xpatch}

\usepackage{thmtools}
\usepackage{thm-restate}
\declaretheorem[name=Theorem]{theorem}
\declaretheorem[name=Lemma]{lemma}
\declaretheorem[name=Assumption]{assumption}
\declaretheorem[name=Condition]{condition}

\declaretheorem[name=Proposition]{proposition}

\makeatletter
  \renewenvironment{proof}[1][Proof]%
  {%
   \par\noindent{\bfseries\upshape {#1.}\ }%
  }%
  {\qed\newline}
  \makeatother

\theoremstyle{definition}  

\newtheorem{corollary}{Corollary}

\theoremstyle{plain}
\newtheorem{definition}{Definition}

\xpatchcmd{\proof}{\itshape}{\normalfont\proofnameformat}{}{}
\newcommand{\proofnameformat}{\bfseries}

\usepackage[capitalize]{cleveref}

\Crefformat{figure}{#2Figure #1#3}
\Crefname{assumption}{Assumption}{Assumptions}
\Crefformat{assumption}{#2Assumption #1#3}

\usepackage{crossreftools}
\usepackage{xparse}

\ExplSyntaxOn
\DeclareDocumentCommand{\XDeclarePairedDelimiter}{mm}
 {
  \__egreg_delimiter_clear_keys: 
  \keys_set:nn { egreg/delimiters } { #2 }
  \use:x 
   {
    \exp_not:n {\NewDocumentCommand{#1}{sO{}m} }
     {
      \exp_not:n { \IfBooleanTF{##1} }
       {
        \exp_not:N \egreg_paired_delimiter_expand:nnnn
         { \exp_not:V \l_egreg_delimiter_left_tl }
         { \exp_not:V \l_egreg_delimiter_right_tl }
         { \exp_not:n { ##3 } }
         { \exp_not:V \l_egreg_delimiter_subscript_tl }
       }
       {
        \exp_not:N \egreg_paired_delimiter_fixed:nnnnn 
         { \exp_not:n { ##2 } }
         { \exp_not:V \l_egreg_delimiter_left_tl }
         { \exp_not:V \l_egreg_delimiter_right_tl }
         { \exp_not:n { ##3 } }
         { \exp_not:V \l_egreg_delimiter_subscript_tl }
       }
     }
   }
 }

\keys_define:nn { egreg/delimiters }
 {
  left      .tl_set:N = \l_egreg_delimiter_left_tl,
  right     .tl_set:N = \l_egreg_delimiter_right_tl,
  subscript .tl_set:N = \l_egreg_delimiter_subscript_tl,
 }

\cs_new_protected:Npn \__egreg_delimiter_clear_keys:
 {
  \keys_set:nn { egreg/delimiters } { left=.,right=.,subscript={} }
 }

\cs_new_protected:Npn \egreg_paired_delimiter_expand:nnnn #1 #2 #3 #4
 {
  \mathopen{}
  \mathclose\c_group_begin_token
   \left#1
   #3
   \group_insert_after:N \c_group_end_token
   \right#2
   \tl_if_empty:nF {#4} { \c_math_subscript_token {#4} }
 }
\cs_new_protected:Npn \egreg_paired_delimiter_fixed:nnnnn #1 #2 #3 #4 #5
 {
  \mathopen{#1#2}#4\mathclose{#1#3}
  \tl_if_empty:nF {#5} { \c_math_subscript_token {#5} }
 }
\ExplSyntaxOff

\XDeclarePairedDelimiter{\supnorm}{
  left=\lVert,
  right=\rVert,
  subscript=\infty
  }

\usepackage{color-edits}
\usepackage{framed}
\usepackage{xargs}
\usepackage{comment}
\usepackage{tikz}
\usetikzlibrary{arrows.meta, positioning, calc}

\newcommand{\cE}{\mathcal{E}}

\newcommand{\cG}{\mathcal{G}}

\newcommand{\cO}{\mathcal{O}}

\newcommand{\cS}{\mathcal{S}}
\newcommand{\cT}{\mathcal{T}}

\newcommand{\wh}[1]{\widehat{#1}}

\newcommand{\Var}{\textnormal{Var}}

\DeclarePairedDelimiter{\norm}{\|}{\|}

\def\medskip{\vskip 10 pt}
\def\bigskip{\vskip 15 pt}

\def\texitem#1{\par\vspace{5pt}
\noindent\hangindent 20pt
\hbox to 20pt {\hss #1 ~}\ignorespaces}

\DeclareMathOperator{\diag}{diag}

\newtheoremstyle{named}{}{}{\itshape}{}{\bfseries}{.}{.5em}{\thmnote{#3's }#1}
\theoremstyle{named}

\makeatletter
\long\def\@makecaption#1#2{
        \vskip 0.8ex
        \setbox\@tempboxa\hbox{\small {\bf #1:} #2}
        \parindent 1.5em  
        \dimen0=\hsize
        \advance\dimen0 by -3em
        \ifdim \wd\@tempboxa >\dimen0
                \hbox to \hsize{
                        \parindent 0em
                        \hfil
                        \parbox{\dimen0}{\def\baselinestretch{0.96}\small
                                {\bf #1.} #2
                                }
                        \hfil}
        \else \hbox to \hsize{\hfil \box\@tempboxa \hfil}
        \fi
        }
\makeatother

\long\def\comment#1{}
\definecolor{battleshipgrey}{rgb}{0.52, 0.52, 0.51}
\definecolor{darkgray}{rgb}{0.66, 0.66, 0.66}
\definecolor{darkgreen}{rgb}{0.0, 0.2, 0.13}
\definecolor{darkspringgreen}{rgb}{0.09, 0.45, 0.27}
\definecolor{dukeblue}{rgb}{0.0, 0.0, 0.61}
\definecolor{olivedrab7}{rgb}{0.24, 0.2, 0.12}
\definecolor{darkblue}{rgb}{0.0, 0.0, 0.55}
\definecolor{darkscarlet}{rgb}{0.34, 0.01, 0.1}
\definecolor{candyapplered}{rgb}{1.0, 0.03, 0.0}
\definecolor{ao(english)}{rgb}{0.0, 0.5, 0.0}
\definecolor{applegreen}{rgb}{0.55, 0.71, 0.0}

\newcommand{\E}{\mathbb E}

\usepackage{algorithm,algorithmic}

\newcommand{\Vstar}{V^\star}

\newcommand{\termState}{\emptyset}

\newcommand{\Vhat}{\wh{V}}
\newcommand{\subG}{\cG}
\newcommandx{\estVal}[1][1=\subG]{\Vhat_{#1}}

\newcommandx{\Vout}[1][1=\subG]{\Vstar_{#1,\mathrm{out}}}

\newcommandx{\estValout}{\Vhat_{\subG,\mathrm{out}}}

\usepackage{xargs,xcolor}
\newcommandx{\subghittingtimet}[1][1=t]{{T_{\subG}(#1)}}

\newcommand{\tarstt}{s_0}
\renewcommand{\star}{*}

\renewcommand{\preceq}{\preccurlyeq}

\usepackage[framemethod=TikZ]{mdframed}

\usepackage{enumitem}

\newcommand{\effhorizon}{h}

\usepackage{array}
\usepackage{algorithm}
\usepackage{tikz}

\providecommand{\Pp}{\mathbb P}
\providecommand{\one}{\mathbf 1}
\providecommand{\R}{\mathbb R}
\providecommand{\adj}{\operatorname{adj}}
\providecommand{\nuMin}{{\nu_{\min}(\subG)}}
\providecommand{\norm}[1]{\left\lVert #1\right\rVert}

\allowdisplaybreaks[1]
\begin{document}

\begin{center}
{\bf{\LARGE{Temporal difference learning adapts to incomplete coverage: an oracle inequality over subgraphs}}}

\vspace*{.2in}
{\large{
 \begin{tabular}{cc}
  Wenlong Mou$^{ \diamond, \star}$ & Jian Qian$^{ \dagger, \star}$ 
 \end{tabular}

}

\vspace*{.2in}

 \begin{tabular}{c}
 Department of Statistical Sciences, University of Toronto$^{\diamond}$
 \end{tabular}
 
  \begin{tabular}{c}
  Department of AI and Data Science, University of Hong Kong$^{ \dagger}$
  \end{tabular}

}

\end{center}
\begin{abstract}
Temporal difference (TD) and Monte Carlo (MC) are two fundamental principles for value
function estimation in reinforcement learning. TD is asymptotically efficient for value estimation, whereas MC requires only adequate target visitation. We show that these advantages are automatically combined by ordinary TD in the tabular setting. We identify the subgraph Bellman variance, arising from bootstrapping within a subgraph and rollouts outside it, as a fundamental quantity connecting the MC and efficient TD benchmarks. With independent trajectories, we prove a non-asymptotic oracle inequality: the mean squared error of full TD is bounded by this
variance divided by the sample size for every sufficiently covered subgraph
containing the target, up to
a higher-order remainder whose occupancy dependence is confined to that
subgraph. No subgraph is supplied to the algorithm. TD thus adapts without selecting a subgraph or requiring global coverage. A finite-sample minimax
lower bound shows that the asymptotic variance of TD need not be attainable before
rare states are
adequately observed. Experimental studies examine the transition between these variance scales and its dependence
on visitation structure.
  \let\thefootnote\relax\footnote{$^{\star}$ WM and JQ contributed equally to this work.}
\end{abstract}

\section{Introduction}\label{sec:introduction}
The dynamic nature of reinforcement learning (RL) makes value functions
central to learning and decision making. A value function has two equivalent
representations: as the fixed point of a Bellman operator and as the expected
reward to go. These representations motivate two fundamental approaches.
\emph{Bootstrapping} estimates values through the empirical Bellman equation,
exploiting the recursive structure of one-step transitions; \emph{rollout}
estimates values by averaging cumulative rewards along observed trajectories.
For policy evaluation, these approaches give temporal difference (TD)
methods~\citep{sutton1988learning,bradtke1996linear} and Monte Carlo (MC)
methods~\citep{curtiss1954theoretical,barto1993monte}, respectively.
Both principles underlie a broad range of RL algorithms; see
\citet{Sutton1998} for a comprehensive treatment.

We study their statistical behavior in a tabular Markov reward process
(MRP), observed through independent complete trajectories. MC has a familiar
advantage in this setting: sum of rewards following the first visit to a target
state is an unbiased observation of its value, so first-visit MC is unbiased
conditional on the target being observed. Averaging these returns gives a
finite-sample guarantee controlled by the number of trajectories that visit
the target, without requiring coverage of other states. Even when much of
the state space is unobserved, MC can reliably estimate a sufficiently
visited target. Its limitation is
variance: using complete returns can be much noisier than pooling transition
information through the Markov structure.

TD exploits precisely this pooling. It plugs empirical transition and reward
estimates into the Bellman equation and, for a fixed finite MRP, attains an
asymptotically efficient variance that can be substantially smaller than the
MC variance~\citep{cheikhi2023statistical}. A standard route to finite-sample
guarantees is to control the Bellman problem globally. Sharp bounds in the
generative setting assume samples from every state~\citep{khamaru2020temporal},
while bounds for asynchronous TD depend on the smallest stationary state
probability~\citep{chen2021lyapunov}. These global guarantees can deteriorate
because of rarely visited states, even when the target itself is well
observed. They leave open whether TD has a useful target-specific guarantee
under incomplete coverage, especially when most states are unobserved.

At a superficial glance, the empirical linear system for TD seems ill-behaved when the states are not fully covered, as the empirical transition matrix is typically a bad approximation of the true one. However, in this paper, we show a strong finite-sample a adaptivity of TD in the incomplete coverage setting --- for any sample size, it achieves non-asymptotic guarantees at least as good as the first-visit MC estimator, while reducing the variance through a local pooling effect.

Our main contribution establishes that ordinary tabular TD adaptively
achieves the \emph{subgraph Bellman variance}. This quantity associates a
variance benchmark with each subgraph containing the target, to describe
the pooling effect locally within the subgraph.
It is the asymptotic variance of an estimator that bootstraps between
successive visits to the subgraph, accumulating rewards until the next
visit or termination. The target singleton gives the first-visit MC
variance, while the full reachable state space gives the efficient TD
variance. Intermediate subgraphs retain transition structure among a
subset of states and provide benchmarks for different levels of coverage.

Theorem~\ref{thm:oracle} formalizes this adaptivity through an oracle
inequality. Informally, with $n$ trajectories,
\begin{align}
 \text{TD error}
 \;\lesssim\;
 \inf_{\substack{\subG\ni\tarstt\\ n\nuMin\gtrsim1}}
 \left\{
   \sqrt{\frac{\sigma_\subG^2(\tarstt)}{n}}
   +\frac{1}{n\nuMin}
 \right\}.\label{eq:informal-oracle}
\end{align}
Here TD error means root mean squared error at the target $\tarstt$,
$\sigma_\subG^2(\tarstt)$ is the subgraph Bellman variance, and
$\nuMin=\min_{s\in\subG}\nu(s)$ is the smallest expected number of visits
per trajectory among states in $\subG$.
The infimum is over deterministic subgraphs. We suppress effective-horizon
and logarithmic factors in both the error bound and the coverage condition.
The same TD estimator satisfies these comparisons without selecting a
subgraph. For the target singleton, this yields the MC variance benchmark
in Proposition~\ref{prop:td-mc}.

Though TD itself is asymptotically efficient, the subgraph Bellman variance in~\Cref{eq:informal-oracle} can still be much larger. We complement the upper bound with a finite-sample lower bound, showing that the efficient full-TD variance need not be attainable
uniformly over large MRPs, even when the target is well visited. In a class
with target occupancy at least $\nu_0$ and exit-probability parameter $q$,
the bound includes a term of order $q/(n\nu_0)$ in addition to the efficient
variance scale. Unresolved excursions into rarely visited states therefore
impose an unavoidable worst-case cost. This supports the role of subgraph
variance as a finite-sample benchmark. Experiments on random graphs, a lattice, and a queue examine
this changing variance scale and its relation to visitation structure.

Section~\ref{sec:problem-setup} introduces the observation model and the TD
and MC estimators, followed by the target-only guarantee in
Subsection~\ref{sec:target-adaptivity}.
Section~\ref{sec:subgraph-variance} motivates and defines the subgraph
variance, then presents the adaptive oracle inequality in
Subsection~\ref{sec:main-upper}. The lower bounds and further experiments appear in
Sections~\ref{sec:lower-bounds} and~\ref{sec:simulations}. Section~\ref{sec:proofs}
contains the main proofs, followed by discussion in
Section~\ref{sec:discussion}. Auxiliary estimates and proofs of the
target-state results are deferred to the appendices.

\Needspace{5\baselineskip}
\subsection{Additional related work}
\paragraph{Theoretical analysis of TD and fixed-point estimation.}
A sequence of earlier works~\citep{sutton1988learning,jaakkola1993convergence,tsitsiklis1997analysis}
establishes the asymptotic convergence of temporal difference and
$Q$-learning algorithms using stochastic approximation.
\citet{cheikhi2023statistical} characterize the variances of TD and MC
estimators and their gap through trajectory pooling. A more recent line of
research~\citep{chen2020finite,chen2021lyapunov,khamaru2020temporal,mou2022optimal}
provides sharp non-asymptotic guarantees for stochastic approximation and
contractive fixed-point problems. Generative-model guarantees provide
observations at each state or state-action pair. Our observation model
instead consists of complete trajectories with random and uneven visitation,
and our result concerns the exact empirical Bellman solution.

\paragraph{Multistep and adaptive methods.}
Multistep methods combine bootstrapping and rollout through weighted
averages, including TD$(\lambda)$ in policy evaluation~\citep{sutton1988learning} and
Q$(\lambda)$ in policy optimization~\citep{peng1994incremental}.
The tuning parameter can be chosen using theoretical
bounds~\citep{kearns2000bias,tsitsiklis1997analysis,chen2021lyapunov,mou2024optimal}
or data-driven heuristics~\citep{white2016greedy,mann2016adaptive,amiranashvili2018td}.
Alternative reweighting schemes for multi-step returns have also been
studied~\citep{konidaris2011td_gamma,thomas2015policy,sharma2017learning,downey2010temporal,daley2024demystifying}.
\citet{riquelme2019adaptive} consider adaptive switching between TD and MC
updates based on state-specific estimates of bias and variance. Our result
establishes adaptivity of ordinary TD itself: the subgraphs index theoretical
risk comparisons, with no corresponding tuning or selection step.

\paragraph{Conditional averaging and statistical efficiency.}
\citet{grunewalder2011optimal} study the optimal unbiased value estimator
and its relation to MC, TD, and LSTD, including a Rao--Blackwell construction
from first-visit MC. For acyclic MRPs, they identify maximum likelihood with
the minimum-variance unbiased estimator under their reward assumptions and
conditional-on-observation formulation. Our oracle inequality applies also
to cyclic absorbing MRPs and compares full TD with every sufficiently
covered subgraph. The proof relates full TD to the conditional average of
a subgraph Bellman estimator, retaining an explicit correction from
singular trajectory reassemblies. Controlling
this correction and the subgraph estimation error yields an oracle inequality
without a coverage requirement outside the comparison subgraph.

\paragraph{Balancing asymptotic efficiency and sample size.}
In many statistical problems, an asymptotically efficient estimator is
difficult to analyze with limited data, whereas a simpler estimator admits
a useful guarantee at a larger variance. Related questions arise in density
estimation~\citep{chan2014near,zhu2020deconstructing}, estimation of treatment
effects~\citep{robins1997toward,mou2022off}, and policy evaluation with
function approximation~\citep{tsitsiklis1997analysis,mou2023optimal}.
Our upper and lower bounds distinguish target-specific reliability from
attainability of the fully pooled variance. The subgraph benchmark describes
how these two properties can be reconciled by the same estimator.

\section{Problem setup and target-state guarantees}\label{sec:problem-setup}
\label{sec:model-and-result}

Let $\cS$ be a finite set of nonterminal states and let $\termState$ be an
absorbing terminal state.  A trajectory is
\[
 (S_0,R_0,S_1,R_1,\ldots,S_{T-1},R_{T-1},S_T=\termState).
\]
Conditional on $S_t=s$, the pair $(R_t,S_{t+1})$ is drawn from a fixed
Markov reward kernel; in particular, the reward may depend jointly on the
current and next state.  The $n$ observed trajectories are independent and
have common initial distribution $\mu$.  We assume, for $\effhorizon\geq1$,
\begin{equation}
 \sup_{s\in\cS}\E_s T^p\leq (p\effhorizon )^p
 \quad(p\in\mathbb N_+),
 \qquad |R_t|\leq1\quad\text{a.s.}
 \label{ass:effective-horizon}
\end{equation}
The value function and primitive Bellman residual are
\[
 \Vstar(s)=\E_s\sum_{t<T}R_t,\qquad \Vstar(\termState)=0,
 \qquad
 \delta_t=R_t+\Vstar(S_{t+1})-\Vstar(S_t).
\]
Thus $|\Vstar|\leq \effhorizon$, $|\delta_t|\leq1+2\effhorizon\leq3\effhorizon$, and
$(\delta_t)$ is a martingale-difference sequence.

Hats denote empirical estimators, bars denote averages, and checks denote
projected auxiliary estimators used only in the proof.

For a vector $w$ indexed by a finite state set, write
$D_w=\diag(w(s))$.  For a positive semidefinite matrix $B$, use the
seminorm $\norm{v}_B^2=v^\top Bv$; it is a norm when $B$ is positive definite.

Let $P$ be the nonterminal transition matrix of the primitive chain and let
$F=(I-P)^{-1}$.  Thus $F(s,u)$ is the expected number of visits to $u$ before
absorption when the chain starts at $s$.  Let
\[
 \nu(s)=\E_\mu\sum_{t<T}\one\{S_t=s\}
\]
be the one-trajectory occupancy, and put
\[
 p_s=\Pp_\mu\{\text{one trajectory visits }s\}.
\]

The occupancy $\nu(s)$ counts all visits, whereas $p_s$ counts trajectories
that visit $s$ at least once. The strong Markov property gives
$\nu(s)=p_sF(s,s)$; in particular, $p_s\geq\nu(s)/h$.
We restrict attention to the reachable states $\{s:\nu(s)>0\}$ when defining
population variance benchmarks. A state with positive occupancy may still
be absent from a finite sample.

\subsection{Temporal difference and Monte Carlo estimation}
\label{sec:preliminary-estimators}
Fix a target state $\tarstt\in\cS$ with $\nu(\tarstt)>0$.  For the complete sample, let
\[
 N_s=\sum_{i=1}^n\sum_{t<T_i}\one\{S_t^{(i)}=s\},
 \quad
 \mathsf C_{su}=\sum_{i=1}^n\sum_{t<T_i}
       \one\{S_t^{(i)}=s,S_{t+1}^{(i)}=u\},
 \quad
 Y_s=\sum_{i,t:S_t^{(i)}=s}R_t^{(i)}.
\]
Here the column index $u$ ranges only over nonterminal states; transitions
to $\termState$ are omitted from $\mathsf C$ but are included in $N_s$.  On the
observed state set $\cO=\{s:N_s>0\}$, put
\[
 \widehat P=D_N^{-1}\mathsf C,\qquad \widehat r=D_N^{-1}Y,
 \qquad
 \widehat V_{\rm TD}=(I-\widehat P)^{-1}\widehat r.
\]
Every observed state has an observed directed path to $\termState$, so this
inverse exists.  If $\tarstt$ is unobserved, we set $\widehat V_{\rm TD}(\tarstt)=0$.
This is the terminal-inclusive, undiscounted plug-in estimator: transitions to
$\termState$ enter every row denominator.  It is not the finite-sample
estimator that conditions on a nonterminal successor and then multiplies by a
known discount factor.

For first-visit MC, let $\tau_i(s)=\inf\{t<T_i:S_t^{(i)}=s\}$, with
$\tau_i(s)=\infty$ if trajectory $i$ does not visit $s$. Put
$H_s=\sum_{i=1}^n\one\{\tau_i(s)<\infty\}$ and define
\begin{equation}
 \widehat V_{\rm MC}(s)
 =\frac{1}{H_s}\sum_{i:\tau_i(s)<\infty}
       \sum_{t=\tau_i(s)}^{T_i-1}R_t^{(i)},
 \qquad H_s>0.
 \label{eq:first-visit-mc}
\end{equation}
Set $\widehat V_{\rm MC}(s)=0$ if $H_s=0$. Each trajectory contributes at
most one return to this estimator. Both estimators therefore use complete
trajectories and the same convention when the target is unobserved.

\subsection{Asymptotic variance and finite-sample reliability}
Define the one-step residual variance
$\sigma_{\Vstar}^2(s)=\E(\delta_t^2\mid S_t=s)$ and the two target variance
benchmarks
\begin{align}
 \sigma_{\rm TD}^2(\tarstt)
 &=\sum_{s:\nu(s)>0}
       \frac{F(\tarstt,s)^2\sigma_{\Vstar}^2(s)}{\nu(s)},
       \label{eq:td-variance}\\
 \sigma_{\rm MC}^2(\tarstt)
 &=\frac{\Var_{\tarstt}(\sum_{t<T}R_t)}{p_{\tarstt}}
 =\frac{1}{p_{\tarstt}}
       \sum_sF(\tarstt,s)\sigma_{\Vstar}^2(s).
       \label{eq:mc-variance}
\end{align}
For a fixed finite MRP, the centered estimators multiplied by $\sqrt n$
converge to normal laws with these variances; see
\citet[Proposition B.4 and Corollary B.6]{cheikhi2023statistical}.
The TD estimator pools transitions from all trajectories that visit the same
state, whereas MC uses only returns following a visit to the target. The
resulting variance inequality
$\sigma_{\rm TD}^2(\tarstt)\leq\sigma_{\rm MC}^2(\tarstt)$ can have a large
gap. Proposition~\ref{prop:lam-lower-bound} gives the local asymptotic minimax
interpretation of the TD benchmark.

MC also has a simple guarantee for every sample size. Since
$H_{\tarstt}\sim\operatorname{Binom}(n,p_{\tarstt})$, conditioning on
$H_{\tarstt}$ yields
\begin{align}
 \E[\widehat V_{\rm MC}(\tarstt)-\Vstar(\tarstt)]^2
 &=\Var_{\tarstt}\!\left(\sum_{t<T}R_t\right)
       \E\frac{\one\{H_{\tarstt}>0\}}{H_{\tarstt}}
       +\Vstar(\tarstt)^2(1-p_{\tarstt})^n\notag\\
 &\leq\frac{2\sigma_{\rm MC}^2(\tarstt)}n
       +h^2e^{-np_{\tarstt}}.
       \label{eq:mc-finite-risk}
\end{align}
The ratio inside the expectation is defined as zero when $H_{\tarstt}=0$.
See Appendix~\ref{subsec:proof-mc-risk} for the elementary calculation.
This bound does not require any occupancy lower bound away from the target.
The question is whether the globally coupled TD equations enjoy an analogous
finite-sample reliability when most other states are poorly observed.
\subsection{Target-state adaptivity without global coverage}
\label{sec:target-adaptivity}

Our first result shows that full TD also admits a finite-sample guarantee
governed by visitation of the target alone. It uses exactly the estimator
defined in Subsection~\ref{sec:preliminary-estimators}, with all observed states included
in the empirical Bellman system. Define the quantities
\[
 \nu_0=\nu(\tarstt),\qquad
 \ell_0=\log\frac{enh}{\nu_0}.
\]

\begin{proposition}[TD with the Monte Carlo variance benchmark]
\label{prop:td-mc}
Under condition~\eqref{ass:effective-horizon}, there is a universal constant
$C>0$ such that, if $n\nu_0\geq Ch^3\ell_0^4$, then, for every
$\eta\in(0,1]$,
\begin{equation}
 \E[\widehat V_{\rm TD}(\tarstt)-\Vstar(\tarstt)]^2
 \leq (1+3\eta)\frac{\sigma_{\rm MC}^2(\tarstt)}n
 +\frac C\eta
   \frac{h^7\ell_0^2\min\{h,\ell_0^3\}}{n^2\nu_0^2}.
 \label{eq:td-mc-risk}
\end{equation}
\end{proposition}
\par\noindent
See Appendix~\ref{subsec:proof-td-mc} for the proof.

Neither the condition nor the remainder contains a minimum occupancy over
the rest of the state space. In particular, the result remains meaningful for MRPs with bounded effective horizon and sufficient target occupancy, even when the number of states is much larger than
$n$.
The guarantee shows that TD at least achieves the finite-sample performance of MC in~\eqref{eq:mc-finite-risk} (up to high-order terms), with comparable sample-size requirements.

Many rows of the empirical transition matrix may have little or no data,
and the nonzero state counts fluctuate with the trajectories.
These difficulties can prevent a uniform approximation of the full
transition kernel. They do not, however, prevent accurate estimation of one target
value. Complete trajectories guarantee that the empirical chain on the
observed support is absorbing, so the TD estimator is well defined even
before all population states have been observed. Proposition~\ref{prop:td-mc}
establishes its statistical reliability in this regime.

This proposition is the singleton case of our main oracle inequality.
It separates two questions: whether TD can estimate a well-visited
target at all, and how far its variance can improve beyond the MC benchmark
as more of the state space becomes informative. The latter question is addressed in the next section.

\section{Subgraph variance and the adaptive oracle inequality}
\label{sec:subgraph-variance}
Proposition~\ref{prop:td-mc} gives a useful guarantee without global coverage,
but it does not explain the improvement over MC as the sample size grows.
We first illustrate this improvement in a random-graph experiment and
then introduce the subgraph Bellman variance as a finite-sample benchmark.
The oracle inequality in Subsection~\ref{sec:main-upper} shows that full TD
adapts to this benchmark for every sufficiently visited deterministic
comparison subgraph, up to an explicit remainder.

\subsection{A motivating experiment}
\label{subsec:motivation-experiment}

The singleton guarantee raises a more quantitative question: which variance
describes TD before its full-model asymptotic approximation becomes accurate?
If an estimator has mean squared error approximately $\sigma^2/n$ over a
range of sample sizes, multiplying its error by $n$ should give an
approximately horizontal curve over that range. The TD curve in
Figure~\ref{fig:motivation-random-graph} has a different shape.

In particular, we use a simple random walk on a randomly generated Barab\'asi--Albert
graph with $10^5$ states, with a uniform initial distribution and independent
termination after each reward. Rewards are Bernoulli, with randomly drawn,
state-specific means held fixed across datasets. The target is the state
with the largest population occupancy.
Section~\ref{sec:simulations} gives the full protocol and further experiments.

Across the sample sizes shown, TD's effective variance $n\,\mathrm{MSE}$
decreases by a factor of about four and approaches its full-model asymptotic
benchmark, while MC remains on a roughly constant, higher scale. TD already
has smaller risk than MC when most states remain unobserved; at the largest
sample size, its MSE is about $1/6$ of MC's.

To analyze this phenomenon, we turn to the subgraph Bellman variance, defined and characterized by a central limit
theorem in Subsection~\ref{subsec:subgraph-bellman}. It is the asymptotic variance of an estimator that uses
Bellman equations on a chosen subgraph and observed returns between visits
to that subgraph. For
$c>0$, define the deterministic comparison sets
\begin{equation}
 G_n^{(c)}=\{s_0\}\cup\{s\in\mathcal S:n\nu(s)\geq c\}.
 \label{eq:simulation-comparison-subgraphs}
\end{equation}
As $n$ increases, these sets include more states. On this graph, their
variances decrease toward the TD benchmark. We use the fixed choice $c=2$
throughout the experiments. The resulting variance curve captures the
changing error scale.

\begin{figure}[t]
\centering
\includegraphics[width=\textwidth]{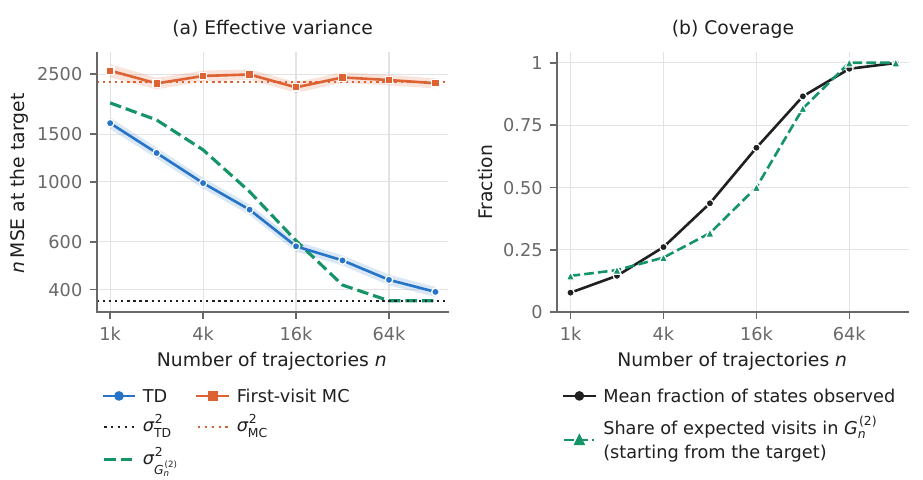}
\caption{A changing finite-sample variance scale on a random graph with
$10^5$ states. Left: empirical $n\,\mathrm{MSE}$ for TD and first-visit MC,
their asymptotic variances, and the subgraph Bellman variance for $G_n^{(2)}$
(defined in Subsection~\ref{subsec:subgraph-bellman}).
Bands show $1.96$ Monte Carlo standard errors on either side of the
empirical means, using $4000$ independent datasets at each sample size.
All variance benchmarks are computed numerically from the population model.
Right: the fraction of all states observed, averaged over datasets, and
the expected number of visits to $G_n^{(2)}$ divided by the expected total
number of visits for a trajectory starting at the target. The latter is
computed from the population model and counts repeated visits and the
initial visit. TD's effective variance decreases as
the comparison subgraph retains more of the target's future.}
\label{fig:motivation-random-graph}
\end{figure}

This correspondence motivates subgraph variance as a complexity measure
for TD's finite-sample performance. In \Cref{subsec:subgraph-bellman}, we describe the subgraph Bellman estimator and its asymptotic variance; and in \Cref{sec:main-upper}, we show that full TD adapts to this variance for every sufficiently visited deterministic comparison subgraph.

\subsection{The subgraph Bellman equation}
\label{subsec:subgraph-bellman}
Fix a deterministic $\subG\subseteq\cS$ with
$\tarstt\in \subG$ and $\nu(s)>0$ for all $s\in \subG$.  Starting from $s\in \subG$, define
\[
 T_\subG=\inf\{t\geq1:S_t\in \subG\cup\{\termState\}\}.
\]
The nonterminal block of the collapsed transition kernel, its fundamental
matrix, and the stopped residual variance are
\begin{align*}
 P_\subG(s,u)&=\Pp_s(S_{T_\subG}=u),\qquad u\in \subG,\\
 F_\subG&=(I-P_\subG)^{-1},\\
 \sigma_{\subG,\Vstar}^2(s)
 &=\Var_s\!\left(\sum_{t<T_\subG}R_t+\Vstar(S_{T_\subG})\right)
   =\E_s\!\left(\sum_{t<T_\subG}\delta_t\right)^2.
\end{align*}
The collapsed mean reward is
$r_\subG(s)=\E_s\sum_{t<T_\subG}R_t$, and the population Bellman equation is
$\Vstar|_\subG=r_\subG+P_\subG\Vstar|_\subG$.
Let $\widehat V_\subG$ denote the subgraph Bellman estimator, obtained
by collapsing each complete trajectory at successive
visits to $\subG\cup\{\termState\}$. We discard the prefix before the first
visit to $\subG$ and discard trajectories that never visit $\subG$.
Retaining terminal transitions in the row
denominators and solving the resulting empirical Bellman system on the
observed collapsed states gives $\widehat V_\subG$. Its target value is set to zero when $\tarstt$ is
unobserved.

We use collapsed to describe the trajectory construction and its kernel.  Write
$\widehat P_\subG$ and $\widehat r_\subG$ for the empirical collapsed kernel
and reward, so that $\widehat V_\subG=\widehat r_\subG+
\widehat P_\subG\widehat V_\subG$ on its observed states.

Define the subgraph Bellman variance by
\begin{equation}
 \sigma_\subG^2(\tarstt)
 :=\sum_{s\in \subG}\frac{F_\subG(\tarstt,s)^2\sigma_{\subG,\Vstar}^2(s)}{\nu(s)}.
 \label{eq:collapsed-leading-variance}
\end{equation}
For a fixed finite MRP and a fixed subgraph $\subG$ as above, under
assumption~\eqref{ass:effective-horizon}, the subgraph Bellman estimator
satisfies the central limit theorem
\begin{equation}
 \sqrt n\bigl(\widehat V_\subG(\tarstt)-\Vstar(\tarstt)\bigr)
 \xrightarrow{\mathrm d}
 \mathcal N\!\left(0,\sigma_\subG^2(\tarstt)\right)
 \qquad\text{as }n\to\infty.
 \label{eq:subgraph-clt}
\end{equation}
Here $n$ counts all original trajectories, including those that never
visit $\subG$. See Subsubsection~\ref{subsubsec:proof-linear-score} for the
derivation. Thus $\sigma_\subG^2(\tarstt)$ is the asymptotic variance of the
subgraph Bellman estimator, motivating its use as a finite-sample benchmark
for full TD.

The collapsed kernel retains the Markov structure between successive visits
to $\subG$ and uses observed returns for the excursions between them.
Excursions can leave and re-enter $\subG$; the construction does not stop
permanently at the first exit. If $\subG=\{\tarstt\}$, its empirical Bellman
equation gives first-visit MC exactly. If $\subG$ is the full reachable
support, every collapsed step is a primitive step and the estimator is full TD.
The corresponding variance identities are
\begin{equation}
 \sigma_{\{\tarstt\}}^2(\tarstt)=\sigma_{\rm MC}^2(\tarstt),\qquad
 \sigma_{\subG_\star}^2(\tarstt)=\sigma_{\rm TD}^2(\tarstt),
 \quad \subG_\star=\{s:\nu(s)>0\}.
 \label{eq:variance-endpoints}
\end{equation}
The singleton identity is proved in Appendix~\ref{subsec:proof-td-mc};
the full-support identity follows because $P_{\subG_\star}$ is the
restriction of $P$ to its reachable support. These subgraphs index a
family of comparison estimators, but our main theorem concerns the full
TD estimator throughout.

The variance cost of collapsing excursions can be bounded in terms of
one-step exit probabilities. For $s\in\subG$, let
$q_\subG^{\rm exit}(s)=\Pp_s\{S_1\in\cS\setminus\subG\}$, excluding transitions
directly to the terminal state. Under~\eqref{ass:effective-horizon},
\begin{equation}
 \sigma_\subG^2(\tarstt)
 \leq \sigma_{\rm TD}^2(\tarstt)
 +4\effhorizon^2\sum_{s\in\subG}
      \frac{F(\tarstt,s)^2}{\nu(s)}q_\subG^{\rm exit}(s).
 \label{eq:subgraph-exit-upper}
\end{equation}
See Appendix~\ref{subsec:proof-subgraph-exit-upper} for the proof.
This upper bound separates the efficient TD variance from an
exit-probability term, weighted by the propagation of errors to the target
and the occupancy of each state in the subgraph.

\subsection{An adaptive oracle inequality for full TD}
\label{sec:main-upper}
The subgraph variance quantifies how much pooling is retained when only the
states in $\subG$ are used as Bellman states. The following theorem shows
that it also controls the risk of ordinary full TD, provided those comparison
states are sufficiently visited.

For every comparison subgraph, write
\[
 \nuMin=\min_{s\in \subG}\nu(s),
 \qquad
 \ell_\subG=\log\frac{en\effhorizon }{\nuMin},
 \qquad
 \mathfrak R_\subG
 =\frac{\effhorizon^7\ell_\subG^2\min\{\effhorizon,\ell_\subG^3\}}
        {n^2\nuMin^2}.
\]

\begin{theorem}[Adaptive subgraph oracle inequality]\label{thm:oracle}
Suppose the effective-horizon condition~\eqref{ass:effective-horizon} holds.
There is a universal constant $C>0$ such that the following holds.
Consider deterministic subgraph $\subG$ satisfying
\begin{equation}
 \tarstt\in\subG\subseteq\cS,
 \qquad \nuMin>0,
 \qquad
 n\nuMin\geq C \effhorizon^3\ell_\subG^4.
 \label{eq:oracle-sample-size}
\end{equation}
For every $\eta\in(0,1]$,
\begin{equation}
 \E\!\left[\bigl|\widehat V_{\rm TD}(\tarstt)-\Vstar(\tarstt)\bigr|^2\right]
 \leq
 \inf_{\subG:\,\text{Eq~}\eqref{eq:oracle-sample-size}}
 \left\{
 (1+3\eta)\frac{\sigma_\subG^2(\tarstt)}{n}
 +\frac{C}{\eta}\mathfrak R_\subG
 \right\}.
 \label{eq:adaptive-oracle-infimum}
\end{equation}
\end{theorem}
\par\noindent
See Subsection~\ref{subsec:proof-adaptive-oracle} for the proof. A few remarks are in order.

First, the
estimator in inequality~\eqref{eq:adaptive-oracle-infimum}
is the same full TD estimator for every
$\subG$; no subgraph is supplied to the estimator. The bound is asymptotically sharp: choosing $\eta$
to tend to zero sufficiently slowly recovers the leading constant one. The oracle inequality exhibits a trade-off between the subgraph variance and the remainder term $\mathfrak R_\subG$. The former decreases as $\subG$ grows, while the latter increases. The infimum over subgraphs identifies the best trade-off for the given sample size. The remainder $\mathfrak R_\subG$ decays with $(n \nu_{\min} (\subG))^{-2}$, which is often dominated by the leading term $\sigma_\subG^2(\tarstt)/n$ for sufficiently visited subgraphs. The logarithmic factors are artifacts of the proof and may be improved with more refined analysis.

The upper bound~\eqref{eq:subgraph-exit-upper} further controls the subgraph
benchmark by the efficient TD variance plus a weighted exit-probability term.

The infimum is over deterministic population subgraphs. For each such
subgraph the bound has the same, subgraph-independent left-hand side, so
one takes a numerical infimum of valid risk bounds. The threshold $\nu(s)\geq c/n$ used in the
experiments is a practical comparison scale; the theorem's sufficient
condition additionally contains the displayed horizon and logarithmic factors.

The proof has two ingredients: we first show that the subgraph Bellman estimator $\widehat V_\subG$ itself achieves the subgraph variance bound~\eqref{eq:adaptive-oracle-infimum} for fixed $\subG$; then we show that the full TD estimator $\widehat V_{\rm TD}$ has risk upper bounded by that of $\widehat V_\subG$ plus high-order remainder terms, for every sufficiently visited subgraph $\subG$. The former step is based on a variant of the leave-one-out analysis in \citet{pananjady2021instance}, while the latter is based on Rao--Blackwellization --- the full TD estimator is approximately the conditional expectation of the subgraph Bellman estimator given sufficient statistics.

\section{Minimax lower bounds}
\label{sec:lower-bounds}

The oracle inequality separates two statistical questions.  The full-TD
variance is the local asymptotic benchmark, but it need not describe the
attainable risk when many states are rarely observed.  We first certify
the benchmark and then give a finite-sample obstruction that applies to
every estimator.

\subsection{The local asymptotic benchmark}

Write $K_s$ for the joint law of $(R_0,S_1)$ given $S_0=s$, including
successors at $\termState$.  Fix a baseline kernel $K^0$ and an initial
distribution $\mu$, and remove states unreachable from $\mu$.  For
$\varepsilon>0$, let $\mathcal N_\varepsilon(K^0)$ consist of kernels on
the same state space satisfying
\[
 K_s\sim K_s^0,
 \qquad
 \max_{s\in\cS}
 \left\|\log\frac{dK_s}{dK_s^0}\right\|_{L^\infty(K_s^0)}
 \leq\varepsilon.
\]
The initial distribution remains fixed in this neighborhood.  Bounded
rewards and the transition support are preserved. We have the following local asymptotic lower bound on the mean squared error of any estimator based on $n$ complete trajectories.

\Needspace{9\baselineskip}
\begin{proposition}[Local asymptotic lower bound]
\label{prop:lam-lower-bound}
Suppose the baseline model satisfies~\eqref{ass:effective-horizon} and
$\nu^0(s)>0$ for every $s\in\cS$.  For every sequence of estimators
$\widehat v_n$ based on $n$ complete trajectories,
\begin{equation}
 \sup_{\Delta>0}\liminf_{n\to\infty}
 \sup_{K\in\mathcal N_{\Delta/\sqrt n}(K^0)}
 n\E_K\bigl[\widehat v_n-V_K(\tarstt)\bigr]^2
 \geq \sigma_{{\rm TD},0}^2(\tarstt),
 \label{eq:local-asymptotic-lower}
\end{equation}
where the variance on the right is evaluated at $K^0$.
\end{proposition}
\par\noindent
See Subsection~\ref{subsec:proof-local-lower} for the proof.

This is the local asymptotic minimax interpretation of the TD variance
\citep{hajek1972local,lecam1973convergence,van2000asymptotic}.
The proof uses a bounded perturbation of the joint reward and successor
law, so it remains within the bounded-reward model.  Together with the
TD central limit theorem, it identifies the asymptotic variance benchmark;
it does not supply a finite-sample optimality guarantee.

\subsection{Finite-sample lower bounds}

Let $d$ denote the number of nonterminal states. Although the optimal
variance in Proposition~\ref{prop:lam-lower-bound}
is asymptotically attained by TD, the available trajectories may not
contain enough information to attain this benchmark in finite samples.
Proposition~\ref{prop:td-mc} shows that TD can achieve the MC benchmark
up to higher-order terms without global coverage, and
Theorem~\ref{thm:oracle} gives a finer
guarantee through the subgraph Bellman variance. These results leave two
questions about the gap from the optimal TD variance:
\begin{enumerate}
 \item Is the asymptotic TD variance always attainable in finite samples? Or do we have to pay for the MC variance when many states are rarely visited?
 \item The upper bound~\eqref{eq:subgraph-exit-upper} controls the subgraph
 Bellman variance by the efficient TD variance plus a weighted
 exit-probability term. Is such an exit-probability dependence unavoidable
 for every estimator under incomplete coverage?
\end{enumerate}
To answer these questions, we characterize a model class by three quantities: the target
occupancy, the TD variance, and the probability of moving from the target
to a state with small occupancy.

For this state space, define
$\mathcal M_d(\nu_0,\sigma_*,\rho,q)$ to be the class of models $M=(\mu,K)$
such that
\begin{align}
 &\text{\eqref{ass:effective-horizon} holds with $\effhorizon=2$},
 \qquad |R_t|\leq1\quad\text{a.s.},\notag\\
 &\nu_M(\tarstt)\geq\nu_0,
 \qquad \sigma_{{\rm TD},M}^2(\tarstt)\leq\sigma_*^2,
 \qquad
 \Pp_{M,\tarstt}\{S_1\in\cS,\ \nu_M(S_1)\leq\rho\}\leq q.
 \label{eq:finite-lower-class}
\end{align}
Besides the bounded-reward and horizon conditions, the class imposes
three restrictions with distinct statistical roles:
\begin{itemize}
 \item The target occupancy is at least $\nu_0$. With the effective
 horizon fixed at two, the MC bound~\eqref{eq:mc-finite-risk} is of order
 $1/(n\nu_0)$ throughout this class.
 \item The asymptotic TD variance is at most $\sigma_*^2$. The regime
 $\sigma_*^2\ll1/\nu_0$ allows a large gap between the efficient variance
 and the worst-case MC scale.
 \item The probability of a one-step transition from the target to a
 nonterminal state with occupancy at most $\rho$ is at most $q$. When
 $n\rho\ll1$, each such state is unlikely to be observed in the sample.
 Thus $q$ controls how strongly the target is connected to this poorly
 observed part of the state space.
\end{itemize}
The terminal state is excluded from the last event. The initial
distribution is part of the model; the lower bound below continues to
hold if it is supplied to the estimator. We fix $\effhorizon=2$ to
isolate the roles of occupancy and variance, without seeking an optimal
dependence on the effective horizon.

\Needspace{12\baselineskip}
\begin{theorem}[Finite-sample lower bound]
\label{thm:finite-sample-lower-bound}
There are universal constants $c,c_1,c_2>0$ such that, for every
$\nu_0>0$, $5<\sigma_*^2<1/\nu_0$, $q\in[0,1]$, and positive integers
$n,d$ satisfying
\[
 n\geq c_1\nu_0^{-2},\qquad d\geq c_2n^2,
 \qquad \rho=2/d,
\]
we have
\begin{equation}
 \inf_{\widehat v_n}
 \sup_{M\in\mathcal M_d(\nu_0,\sigma_*,\rho,q)}
 \E_M\bigl[\widehat v_n-V_M(\tarstt)\bigr]^2
 \geq \frac c n\left(\sigma_*^2+\frac q{\nu_0}\right).
 \label{eq:finite-sample-lower}
\end{equation}
The infimum ranges over all estimators based on the $n$ trajectories.
\end{theorem}
\par\noindent
See Subsection~\ref{subsec:proof-finite-lower} for the proof.

A few remarks explain how the theorem answers the two questions.
First, it applies in the regime
$\nu_0^{-2}\lesssim n\lesssim\sqrt d$: the sample is large enough to
observe the target repeatedly, but too small to cover the entire state
space. With $\rho=2/d\ll1/n$, each low-occupancy state is unlikely to
be visited, even though these states may collectively receive substantial
probability. The theorem covers a narrower range than
$\nu_0^{-1}\lesssim n\lesssim d$, where the target is visited but
global coverage may still be insufficient.

Second, the lower bound has the intuitive form
\[
 \text{minimax MSE}
 \;\gtrsim\;\frac1n\left(
 \text{optimal TD variance}
 +\frac{\text{exit probability}}{\text{target visitation measure}}
 \right).
\]
The two terms correspond to the class parameters $\sigma_*^2$ and
$q/\nu_0$. Here ``exit probability'' refers to one-step transitions from
the target into the low-occupancy nonterminal states. If the constraint
on this probability is removed by setting $q=1$, the minimax risk is at
least of order $1/(n\nu_0)$, even when $\sigma_*^2\ll1/\nu_0$.
This matches the worst-case MC scale. Hence a small asymptotic TD
variance does not by itself guarantee a corresponding finite-sample
improvement: the information needed to exploit trajectory pooling may
be absent from the observed trajectories.

Finally, the parameter $q$ gives a more refined description of this
obstruction. The additional cost $q/(n\nu_0)$ is unavoidable in the
worst case over~\eqref{eq:finite-lower-class}, for every estimator.
In the regime where this additional term dominates, the difficult family
used in the proof has the target enter the rare
layer with probability $q/2$, while the full-TD asymptotic variance
remains bounded by five. Thus the extra term reflects an
information-theoretic gap, rather than a limitation specific to TD or
MC. The same cost appears in the subgraph Bellman variance of this
family; see~\eqref{eq:lower-family-collapsed-variance} in the proof.
This connection explains why adaptation to sufficiently visited
subgraphs can coexist with a gap from the full-TD variance.

\section{Simulation studies}
\label{sec:simulations}

We examine how TD's finite-sample error depends on state-space size and
visitation structure, and compare it with an explicit subgraph estimator.
The experiments address two aspects of the oracle inequality: the changing
variance scale before the full-model approximation becomes accurate, and
the possibility of estimating the target efficiently despite incomplete
global coverage.

\subsection{Models and protocol}

Table~\ref{tab:simulation-models} summarizes the four absorbing MRP families.
The first three use simple random walks and Bernoulli rewards with
state-specific means drawn independently from $\operatorname{Uniform}(0,1)$
and held fixed across datasets. The queue has reward $(x_1+x_2)/80$, where
$x_j$ is the number of jobs at station $j$.

\begin{table}[ht]
\centering
\caption{MRP families used in the experiments. Termination occurs independently
 after each reward; queue event probabilities are conditional on survival.}
\label{tab:simulation-models}
\fontsize{10}{13}\selectfont
\setlength{\tabcolsep}{4pt}
\renewcommand{\arraystretch}{1.12}
\begin{tabular}{@{}>{\raggedright\arraybackslash}p{0.17\textwidth}
 >{\raggedright\arraybackslash}p{0.115\textwidth}
 >{\raggedright\arraybackslash}p{0.315\textwidth}
 >{\raggedright\arraybackslash}p{0.20\textwidth}
 >{\raggedright\arraybackslash}p{0.12\textwidth}@{}}
\hline
Model & States $d$ & Graph or queue parameters & Initial distribution & Stop probability \\
\hline
Scale-free & $10^3$, $10^4$, $10^5$ & Barab\'asi--Albert; attachment $3$
 & Uniform & $0.1$ \\
Small-world & $10^4$ & Connected Watts--Strogatz; initial degree $6$, rewiring $0.1$
 & Uniform & $0.1$ \\
Lattice & $10^4$ & $100\times100$ grid; nearest neighbors, no wraparound
 & Uniform on a $10\times10$ corner block & $0.01$ \\
Tandem queue & $1681$ & Two buffers of size $40$; arrival $0.3$, service $0.35$ per station
 & Uniform & $0.01$ \\
\hline
\end{tabular}
\end{table}

\paragraph{Sampling.}
The main target maximizes population occupancy. Let $p_{s_0}$ be its per-trajectory
hitting probability. We use $n=2^k\lceil10/p_{s_0}\rceil$, with $k=0,\ldots,7$
for the scale-free graphs and $k=0,\ldots,4$ otherwise, and $1000$ independent
datasets of complete trajectories per point. The motivating experiment uses
$n=1000\cdot2^k$, $k=0,\ldots,7$, and $4000$ datasets per point.
Targets and sample sizes depend only on the population model.

\paragraph{Estimators and benchmarks.}
TD and first-visit MC use the same datasets and return zero when the target
is unobserved. The motivating and four-structure experiments also evaluate
the subgraph estimator on the same data, with the same zero convention.
All subgraph comparisons use the deterministic population sets $G_n^{(2)}$
in \eqref{eq:simulation-comparison-subgraphs}. Variance benchmarks are
computed from the population model, independently of the estimation errors.
Appendix~\ref{app:simulation-details} gives boundary conventions, seeds,
estimator details, and the interpretation of the cutoff $c=2$.

\subsection{The effect of state-space size}

We generate five Barab\'asi--Albert graphs at each size
$d\in\{10^3,10^4,10^5\}$.
Figure~\ref{fig:simulation-size} plots $n\,\mathrm{MSE}/\sigma_{\rm TD}^2$
against the expected number $np_{s_0}$ of target-hitting trajectories,
averaging both coordinates across instances at each sample-size doubling
stage. This normalization compares instances at similar levels of target
information. At the first stage, the mean normalized TD errors are
$1.82$, $3.21$, and $4.92$, respectively. The corresponding mean subgraph
variance ratios are $1.97$, $3.86$, and $5.57$.
The larger graphs exhibit a higher and longer transition toward the
full-model benchmark, while MC remains on a larger variance scale.

\begin{figure}[t]
\centering
\includegraphics[width=\textwidth]{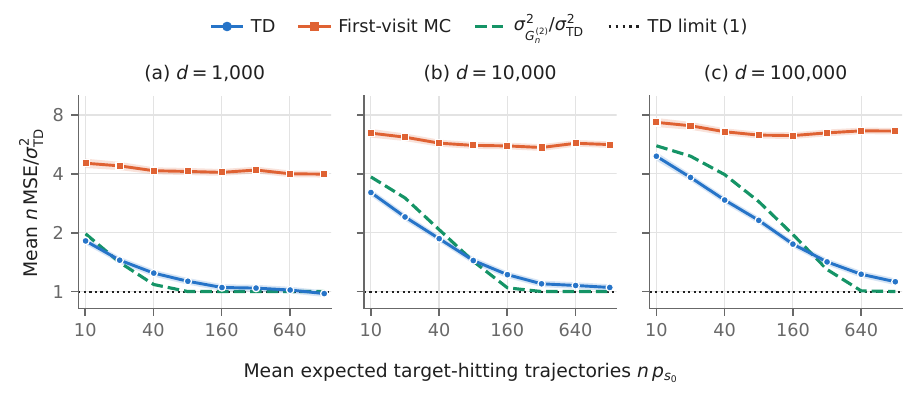}
\caption{State-space size and the approach to the TD asymptotic variance.
Each panel averages five Barab\'asi--Albert instances, with the most
visited state as target and $1000$ datasets per sample size and instance.
Solid curves show mean instance-normalized errors for TD and first-visit
MC; the dashed curve shows the mean population subgraph variance ratio
for $c=2$. Both coordinates are averaged at common doubling stages.
Shading gives pointwise $95\%$ Monte Carlo confidence intervals conditional
on the fixed instances (Appendix~\ref{app:simulation-details}). The dotted
line is the TD benchmark $1$. Both axes are logarithmic.}
\label{fig:simulation-size}
\end{figure}

The average over graph instances shows that the extended transition in
Subsection~\ref{subsec:motivation-experiment} is not confined to a single
graph realization. It also illustrates the distinction behind the
finite-sample lower bound: reliable target estimation need not imply an
accurate full-model variance approximation. These random graphs are not
the lower-bound construction, and the experiment does not test its
minimax rate.

\subsection{Different visitation structures}

Figure~\ref{fig:simulation-structures} compares three Barab\'asi--Albert
graphs with $10^4$ states, three Watts--Strogatz graphs, three reward
realizations on the same lattice, and one tandem queue.
The most visited target on the scale-free graphs exhibits a pronounced
transition. For the other three model families, the comparison subgraph
variance is already close to the full-TD variance at the first displayed
sample size. TD is correspondingly within about $1.0$--$1.3$ times that
benchmark at the first point, whereas MC has a substantially larger error.

\begin{figure}[t]
\centering
\includegraphics[width=\textwidth]{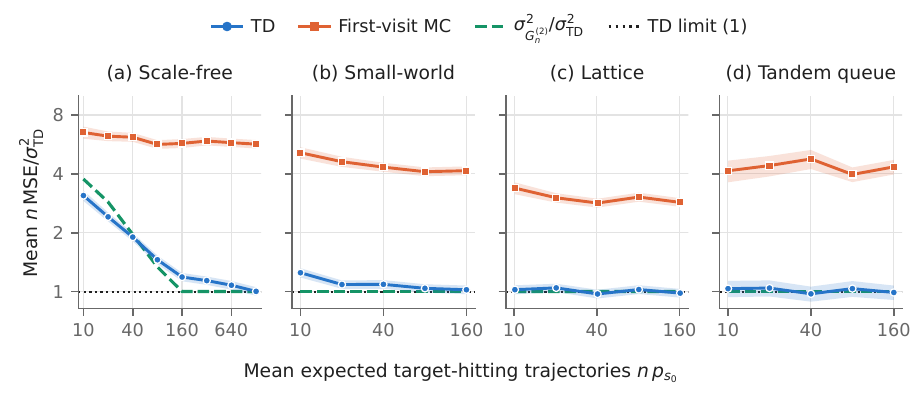}
\caption{Finite-sample errors under four visitation structures.
The target is the most visited state. The normalization, curve encodings,
and $1000$-dataset replication protocol are the same as in
Figure~\ref{fig:simulation-size}. Each method has one mean curve with a
pointwise $95\%$ Monte Carlo confidence band.
The lattice instances share a transition kernel and differ in reward
means; the queue has one instance. A long transition occurs for the
scale-free targets, while the other models approach the TD benchmark
earlier.}
\label{fig:simulation-structures}
\end{figure}

The lattice separates global coverage from information relevant to the
target particularly clearly. At the first sample size, $n=45$, only
about $4.4\%$ of its $10^4$ states are observed. Yet $G_n^{(2)}$ contains
$339$ states and retains about $95.1\%$ of the target's expected future
visits; its variance is within $0.01\%$ of the full-TD benchmark.
The observed TD error is already close to this benchmark. By contrast,
the small-world and queue datasets observe almost all states at their
first sample sizes. Thus low global coverage alone does not determine
the difficulty: the relation between visitation and the target's future
also matters. Since the families differ in their initial distributions,
termination probabilities, and rewards, these comparisons demonstrate
behavior across models rather than isolate the effect of topology alone.

\subsection{Comparison with explicit subgraph estimation}

The variance curves describe population benchmarks. To compare actual
estimators, we evaluate full TD and the subgraph Bellman estimator on
$G_n^{(2)}$ using the same trajectories.
In the motivating graph, the subgraph estimator's MSE is
$1.29$--$1.56$ times that of TD for $n\leq16000$; the ratio decreases
to $1$ when the comparison set becomes the full state space.
Table~\ref{tab:subgraph-comparison} summarizes this comparison for the
motivating graph and the four structures, separating points at which
the population comparison set is a proper subset of the state space.

\begin{table}[t]
\centering
\caption{MSE of the explicit subgraph estimator relative to full TD,
using $c=2$ and the most visited target.
The second column counts design points with $G_n^{(2)}\ne\mathcal S$
out of all points in the indicated experiment.
Medians and ranges summarize only those proper-subgraph points,
across instances and sample sizes; they are not confidence intervals.
The other points have $G_n^{(2)}=\mathcal S$, so the estimators coincide.
Each ratio uses paired datasets: $4000$ per point in the first row and
$1000$ in the remaining rows.}
\label{tab:subgraph-comparison}
\begin{tabular}{lccc}
\hline
Model & Proper subgraph / total & Median & Range \\
\hline
Scale-free, $d=10^5$ & $6/8$ & $1.4414$ & $[1.0792,\,1.5585]$ \\
Scale-free, $d=10^4$ & $13/24$ & $1.4204$ & $[0.9999,\,1.5660]$ \\
Small-world & $0/15$ & \multicolumn{2}{c}{$\widehat V_G=\widehat V_{\rm TD}$} \\
Lattice & $15/15$ & $1.0000$ & $[0.9999,\,1.0004]$ \\
Tandem queue & $0/5$ & \multicolumn{2}{c}{$\widehat V_G=\widehat V_{\rm TD}$} \\
\hline
\end{tabular}
\end{table}

These experiments show no material gain from explicit subgraph
restriction in the reported settings. They support the role of
subgraph variance as a comparison for full TD, without suggesting that
TD must select a subgraph to obtain the benefit.

Across individual instances in the main experiments, the observed ratio
$n\,\mathrm{MSE}_{\rm TD}/\sigma_{G_n^{(2)}}^2$ ranges from $0.72$ to
$1.30$. The fixed cutoff captures the variance scale without resolving
all finite-sample differences. In the motivating experiment, squared
empirical bias accounts for at most $0.07\%$ of the empirical MSE.
This diagnostic indicates a predominantly variance-driven transition
at the available replication precision.

\section{Proofs}\label{sec:proofs}\label{sec:proof-adaptive-oracle}
We first prove the oracle inequality and its two risk comparisons, then
establish the information-theoretic lower bounds. The Monte Carlo risk
bound and target-state guarantee are proved in
Appendix~\ref{app:miscellaneous-proofs}.
Throughout the proofs, $C$ denotes a positive universal constant that may
change from line to line.
\setcounter{secnumdepth}{4}

Enlarging the single universal constant in Theorem~\ref{thm:oracle}
makes all displayed uses compatible.

\subsection{Proof of Theorem~\ref{thm:oracle}}
\label{subsec:proof-adaptive-oracle}

The proof uses two propositions.  The first gives a sharp target-coordinate
bound for the subgraph Bellman estimator.

\begin{proposition}[Pointwise risk of the subgraph Bellman estimator]
\label{prop:collapsed}
Under condition~\eqref{eq:oracle-sample-size}, for every $\eta\in(0,1]$,
\begin{equation}
 \E\bigl[\widehat V_\subG(\tarstt)-\Vstar(\tarstt)\bigr]^2
 \leq(1+\eta)\frac{\sigma_\subG^2(\tarstt)}n
 +C(1+\eta^{-1})\mathfrak R_\subG.
 \label{eq:collapsed-plugin-risk}
\end{equation}
\end{proposition}
\par\noindent
See Subsection~\ref{subsec:proof-collapsed} for the proof.

The second proposition transfers this bound to the ordinary full TD estimator
without imposing an occupancy condition outside $\subG$.

\begin{proposition}[Full-to-subgraph risk comparison]
\label{prop:bridge}
Under condition~\eqref{eq:oracle-sample-size}, for every $\eta\in(0,1]$,
\begin{equation}
 \E[\widehat V_{\rm TD}(\tarstt)-\Vstar(\tarstt)]^2
 \leq(1+\eta)
       \E[\widehat V_\subG(\tarstt)-\Vstar(\tarstt)]^2
 +C(1+\eta^{-1})\mathfrak R_\subG.
 \label{eq:full-to-collapsed-risk}
\end{equation}
\end{proposition}
\par\noindent
See Subsection~\ref{subsec:proof-bridge} for the proof.

Taking Propositions~\ref{prop:collapsed} and~\ref{prop:bridge} as given,
apply Proposition~\ref{prop:bridge} with parameter $\eta$, followed by
Proposition~\ref{prop:collapsed} with the same parameter.  Since
$(1+\eta)^2\leq1+3\eta$ for $\eta\leq1$, all higher-order terms are bounded by
$C\eta^{-1}\mathfrak R_\subG$.  Thus the full-TD risk is at most
$(1+3\eta)\sigma_\subG^2(\tarstt)/n+C\eta^{-1}\mathfrak R_\subG$.

The inequality holds separately for every deterministic $\subG$ satisfying
condition~\eqref{eq:oracle-sample-size}, while
its left-hand side does not depend on $\subG$.  Taking the numerical infimum of
the right-hand sides proves inequality~\eqref{eq:adaptive-oracle-infimum} and
completes the proof of Theorem~\ref{thm:oracle}.

\subsection{Proof of Proposition~\ref{prop:collapsed}}
\label{subsec:proof-collapsed}

Throughout this proof, let
\[
 A=D_\nu(I-P_\subG),\qquad a=A^{-\top}e_{\tarstt}.
\]
The proof separates the linear term from the interaction between the
empirical inverse and the empirical Bellman residual.

Collapse each observed trajectory at its successive visits to
$\subG\cup\{\termState\}$.  A trajectory that never visits $\subG$ contributes an
empty collapsed path.  Otherwise write it as
\[
 X_{i,0},\ldots,X_{i,M_i-1}\in \subG,\qquad X_{i,M_i}=\termState,
\]
where a collapsed reward contains all primitive rewards until the next
displayed state.  With $e_s$ the coordinate vector for $s\in\subG$ and
$e_{\termState}=0$, define
\begin{align*}
 L_i&=\sum_{t<M_i}e_{X_{i,t}}
       (e_{X_{i,t}}-e_{X_{i,t+1}})^\top,\\
 \varepsilon_i&=\sum_{t<M_i}e_{X_{i,t}}\xi_{i,t},
 \qquad
 \xi_{i,t}=\sum_{u\text{ in segment }(i,t)}\delta_{i,u},\\
 \widehat A&=\frac1n\sum_{i=1}^nL_i,\qquad
 \bar\varepsilon=\frac1n\sum_{i=1}^n\varepsilon_i.
\end{align*}

Let $\widehat\Delta_\subG(\tarstt)=\widehat V_\subG(\tarstt)-\Vstar(\tarstt)$ on the whole sample space.
Whenever every state in $\subG$ is observed, define
$\widehat a=\widehat A^{-\top}e_{\tarstt}$.  The empirical Bellman equation gives
\begin{equation}
 \widehat V_\subG-\Vstar|_\subG=\widehat A^{-1}\bar\varepsilon.
 \label{eq:collapsed-plugin-error}
\end{equation}
Consequently, on this event,
\begin{equation}
 \widehat\Delta_\subG(\tarstt)=\underbrace{a^\top\bar\varepsilon}_{Z_\subG}
       +(\widehat a-a)^\top\bar\varepsilon.
 \label{eq:subgraph-linearization}
\end{equation}
The first term contains the target variance.  The second term is bilinear
in the inverse perturbation and the empirical residual; we control it
without replacing the inner product by a product of full-vector norms.

Define the trajectory-score covariance
\[
 \Lambda_\subG^\star
 =\diag\bigl(\nu(s)\sigma_{\subG,\Vstar}^2(s)\bigr)_{s\in\subG}.
\]
The following lemma identifies the second moment of the linear term.

\begin{lemma}[Linear term and asymptotic normality]\label{lem:linear-score}
The trajectory scores satisfy
\begin{equation}
 \E L_i=A,\qquad \E \varepsilon_i=0,\qquad
 \E[\varepsilon_i\varepsilon_i^\top]=\Lambda_\subG^\star.
 \label{eq:collapsed-score-moments}
\end{equation}
In particular,
\begin{equation}
 \E Z_\subG^2=\frac1n a^\top \Lambda_\subG^\star a
 =\frac{\sigma_\subG^2(\tarstt)}n.
 \label{eq:linear-score-variance}
\end{equation}
For a fixed finite MRP and a fixed deterministic subgraph $\subG$ containing
$\tarstt$ with $\nu(s)>0$ for every $s\in\subG$, under
assumption~\eqref{ass:effective-horizon}, the subgraph Bellman estimator
satisfies the central limit theorem~\eqref{eq:subgraph-clt} as $n\to\infty$.
\end{lemma}
\par\noindent
See Subsubsection~\ref{subsubsec:proof-linear-score} for the proof.

The next lemma controls the bilinear remainder on an event where the
empirical Bellman system is well defined.  Its proof separates residual
control from the leave-one-out and leave-two-out estimates.

\begin{lemma}[Bilinear remainder]\label{lem:bilinear-remainder}
Under condition~\eqref{eq:oracle-sample-size}, there is an event $\cE$ such
that
\begin{equation}
 \Pp(\cE^c)\leq(en\effhorizon /\nuMin)^{-24},
 \label{eq:bilinear-event-probability}
\end{equation}
every state of $\subG$ is observed on $\cE$, and
\begin{equation}
 \E\bigl[((\widehat a-a)^\top\bar\varepsilon)^2\one_{\cE}\bigr]
 \leq C\mathfrak R_\subG.
 \label{eq:bilinear-remainder-risk}
\end{equation}
The integrand in~\eqref{eq:bilinear-remainder-risk} is defined to be zero
outside $\cE$.
\end{lemma}
\par\noindent
See Subsubsection~\ref{app:collapsed-technical} for the proof.

Taking Lemmas~\ref{lem:linear-score} and~\ref{lem:bilinear-remainder} as
given, apply Young's inequality to
decomposition~\eqref{eq:subgraph-linearization} on $\cE$.  This gives
\[
 \E[\widehat\Delta_\subG(\tarstt)^2\one_{\cE}]
 \leq(1+\eta)\frac{\sigma_\subG^2(\tarstt)}n
       +C(1+\eta^{-1})\mathfrak R_\subG.
\]

The exceptional-event argument uses a pathwise inverse bound and its moment
consequences.  The same bounds also control the full-TD error in the bridge.

\begin{lemma}[Empirical resistance bound]\label{lem:tail}
Let $L=D_N-\mathsf C$ be the integer path Laplacian of any collection of complete
finite paths on $d$ observed states.  Then
\begin{equation}
 0\leq(L^{-1})_{\tarstt s}\leq(L^{-1})_{ss}\leq2(d+1).
 \label{eq:path-laplacian-resistance}
\end{equation}
Consequently, if $M_{\rm tot}=\sum_iT_i$, then
\begin{align}
 \norm{\widehat V_{\rm TD}(\tarstt)-\Vstar(\tarstt)}_{L^4}
 &\leq Cn^2\effhorizon^3,\label{eq:full-td-fourth-moment}\\
 \norm{\widehat V_\subG(\tarstt)-\Vstar(\tarstt)}_{L^4}
 &\leq \effhorizon+\frac{Cn\effhorizon^3}{\nuMin}.
 \label{eq:collapsed-fourth-moment}
\end{align}
\end{lemma}
\par\noindent
See Appendix~\ref{app:proof-empirical-resistance} for the proof.

It remains to control the exceptional event.  Lemma~\ref{lem:tail},
inequality~\eqref{eq:collapsed-fourth-moment}, and H\"older's inequality give
\[
 \E[\widehat\Delta_\subG(\tarstt)^2\one_{\cE^c}]
 \leq\norm{\widehat\Delta_\subG(\tarstt)}_{L^4}^2\Pp(\cE^c)^{1/2}
 \leq C\left(\effhorizon^2+\frac{n^2\effhorizon^6}{\nuMin^2}\right)
          (en\effhorizon /\nuMin)^{-12}
 \leq C\mathfrak R_\subG.
\]
For the last inequality, use $\nuMin\leq \effhorizon$, $n\leq en\effhorizon /\nuMin$, and
$n\nuMin\geq C\effhorizon^3\ell_\subG^4$.  The negative power of $en\effhorizon /\nuMin$ absorbs both
terms defining $\mathfrak R_\subG$.  Adding the two event contributions proves
inequality~\eqref{eq:collapsed-plugin-risk}, completing the proof of
Proposition~\ref{prop:collapsed}.

\subsubsection{Proof of Lemma~\ref{lem:linear-score}}
\label{subsubsec:proof-linear-score}

The expected number of visits to $s$ in one trajectory is $\nu(s)$, so
$\E L_i=D_\nu(I-P_\subG)=A$.  The stopped Bellman identity makes the successive
collapsed residuals martingale differences.  Their conditional variances
are $\sigma_{\subG,\Vstar}^2(X_{i,t})$.  Distinct increments are orthogonal, including those
belonging to different state coordinates; hence $\E \varepsilon_i=0$ and
$\E[\varepsilon_i\varepsilon_i^\top]=\Lambda_\subG^\star$.

The trajectories are independent.  Therefore
\[
 \E(a^\top\bar\varepsilon)^2
 =\frac1n a^\top \Lambda_\subG^\star a
 =\frac1n\sum_{s\in \subG}
       \frac{F_\subG(\tarstt,s)^2\sigma_{\subG,\Vstar}^2(s)}{\nu(s)},
\]
where $a^\top=e_{\tarstt}^\top F_\subG D_\nu^{-1}$ by the definition of $A$.
The last expression equals
$\sigma_\subG^2(\tarstt)/n$ by definition~\eqref{eq:collapsed-leading-variance}.

These identities also yield the central limit theorem~\eqref{eq:subgraph-clt}.
Fix the MRP and $\subG$ as $n\to\infty$. The trajectory pairs
$(L_i,\varepsilon_i)$ are independent and identically distributed, including
zero contributions from trajectories that never visit $\subG$.
Assumption~\eqref{ass:effective-horizon} ensures the required moments are
finite. The law of large numbers and the multivariate central limit theorem,
together with~\eqref{eq:collapsed-score-moments}, give
\[
 \widehat A\xrightarrow{\Pp}A,
 \qquad
 \sqrt n\,\bar\varepsilon
 \xrightarrow{\mathrm d}\mathcal N(0,\Lambda_\subG^\star).
\]
The matrix $A=D_\nu(I-P_\subG)$ is invertible because $\nu(s)>0$ on
$\subG$ and the collapsed chain is transient. Moreover,
$\Pp\{\text{some }s\in\subG\text{ is unobserved}\}
\leq\sum_{s\in\subG}(1-p_s)^n\to0$.
Applying Slutsky's theorem to identity~\eqref{eq:collapsed-plugin-error}
therefore gives the target-coordinate limit with variance
$a^\top\Lambda_\subG^\star a=\sigma_\subG^2(\tarstt)$, proving
convergence~\eqref{eq:subgraph-clt}.
This proves Lemma~\ref{lem:linear-score}.

\subsubsection{Proof of Lemma~\ref{lem:bilinear-remainder}}
\label{app:collapsed-technical}

We retain the notation from Subsection~\ref{subsec:proof-collapsed}.
The bilinear remainder is controlled by an exact leave-one-out
decomposition.  We first state the empirical regularity and residual
estimates needed to complete this argument; their proofs and the
leave-one-out and leave-two-out perturbation estimates follow below.

For $I\subseteq[n]$, let $N_s^{(-I)}$,
$\widehat P_\subG^{(-I)}$, and $\widehat F_\subG^{(-I)}=(I-\widehat P_\subG^{(-I)})^{-1}$ denote the empirical
visit counts, collapsed kernel, and fundamental matrix after deleting the
trajectories in $I$, with the inverse interpreted whenever it exists; every
row denominator remains terminal-inclusive.  Set
\[
 \widehat A^{(-I)}=\frac1n\sum_{i\notin I}L_i.
\]

\begin{lemma}[Uniform full/delete-one/delete-two event]\label{lem:good}
There exist universal constants $C_0,C_1>0$ such that, with
$L_0=C_0\effhorizon\ell_\subG$, under condition~\eqref{eq:oracle-sample-size} there is an event
$\cE$ with
\begin{equation}
 \Pp(\cE^c)\leq (en\effhorizon /\nuMin)^{-24}
 \label{eq:good-event-probability}
\end{equation}
such that the following statements hold simultaneously after deleting any
set $I$ of at most two trajectories:
\begin{enumerate}
\item\label{item:good-length}
      every primitive trajectory has length at most $L_0$;
\item\label{item:good-counts}
      $\frac12n\nu(s)\leq N_s^{(-I)}\leq2n\nu(s)$ for every
      $s\in \subG$;
\item\label{item:good-horizon}
      $\widehat P_\subG^{(-I)}$ is transient and $\widehat F_\subG^{(-I)}\one\leq2\effhorizon\one$;
\item\label{item:good-inverse}
\begin{equation}
 \norm{D_\nu^{1/2}(\widehat A^{(-I)})^{-1}D_\nu^{1/2}}_{2\to2}\leq C_1\effhorizon.
 \label{eq:delete-two-inverse-bound}
\end{equation}
\end{enumerate}
\end{lemma}
\par\noindent
See Appendix~\ref{subsec:proof-good} for the proof.

The influence vector satisfies
\begin{equation}
 a_s=\frac{F_\subG(\tarstt,s)}{\nu(s)},\qquad
 \norm{a}_{D_\nu}^2=\sum_{s\in \subG}\frac{F_\subG(\tarstt,s)^2}{\nu(s)}
 \leq\frac{\effhorizon^2}{\nuMin}.
 \label{eq:population-influence-norm}
\end{equation}
Indeed, after the first visit to $\tarstt$, the original chain contributes
$F_\subG(\tarstt,s)$ expected visits to $s$.  Therefore
$\nu(s)\geq p_{\tarstt}F_\subG(\tarstt,s)$, and identity~\eqref{eq:occupancy-hit-probability}
yields
\[
 \sum_s\frac{F_\subG(\tarstt,s)^2}{\nu(s)}
 \leq p_{\tarstt}^{-1}\sum_sF_\subG(\tarstt,s)
 \leq \effhorizon^2/\nu(\tarstt)\leq \effhorizon^2/\nuMin.
\]

For each deleted set $I\subseteq[n]$ with $|I|\leq2$, define
\[
 \widehat a^{(-I)}=(\widehat A^{(-I)})^{-\top}e_{\tarstt}
\]
whenever the inverse exists.  In particular,
$\widehat a^{(-\varnothing)}=\widehat a$; the superscripts $(-i)$ and
$(-i,-j)$ mean deletion of $\{i\}$ and $\{i,j\}$, respectively.

To use these vectors in unconditional expectations, take the raw value to
be $a$ when the inverse does not exist, and project it in the $D_\nu$ norm
onto the deterministic ball
\[
 \mathbb B=\{v:\norm{v-a}_{D_\nu}\leq C_2\effhorizon/\sqrt{\nuMin}\}.
\]
Here $C_2>C_1+1$ is a fixed universal constant.  Write $\check a^{(-I)}$ for the
projected vector.  The check accent denotes stabilization.  These auxiliary vectors are used
only in the proof; the estimator itself uses all trajectories.
Lemma~\ref{lem:good} makes the raw
and projected values identical on $\cE$.  Projection is nonexpansive, and
$\check a^{(-I)}$ depends only on its training trajectories.  Indeed, on $\cE$,
\[
 \norm{\widehat a^{(-I)}}_{D_\nu}
 =\norm{D_\nu^{1/2}(\widehat A^{(-I)})^{-\top}e_{\tarstt}}_2
 \leq
 \norm{D_\nu^{1/2}(\widehat A^{(-I)})^{-\top}D_\nu^{1/2}}_{2\to2}
 \norm{D_\nu^{-1/2}e_{\tarstt}}_2
 \leq\frac{C_1\effhorizon}{\sqrt{\nu(\tarstt)}},
\]
and inequality~\eqref{eq:delete-two-inverse-bound} also bounds the transposed
normalized inverse.  Together with
inequality~\eqref{eq:population-influence-norm}, this places
$\widehat a^{(-I)}$ in $\mathbb B$ once $C_2>C_1+1$.

Define the stabilized cross-fitted term
\[
 R_n^{\rm loo}=\frac1n\sum_{i=1}^n(\check a^{(-i)}-a)^\top \varepsilon_i.
\]
The following lemma gives the target-coordinate decomposition and controls the
same-sample term by aggregate leverage.

\begin{lemma}[Deterministic aggregate leverage]\label{lem:leverage}
Under condition~\eqref{eq:oracle-sample-size}, on $\cE$ there is an exact
decomposition
\begin{equation}
 \widehat\Delta_\subG(\tarstt)-Z_\subG=R_n^{\rm loo}+R_n^{\rm self},
 \label{eq:collapsed-master-decomposition}
\end{equation}
where
\begin{equation}
 |R_n^{\rm self}|\leq\frac{C\effhorizon^4\ell_\subG}{n\nuMin}.
 \label{eq:same-sample-remainder-bound}
\end{equation}
Set $R_n^{\rm self}=0$ outside $\cE$.
\end{lemma}
\par\noindent
See Paragraph~\ref{subsubsec:proof-leverage} for the proof.

The deterministic estimate controls the accumulated absolute residual.
The next lemma uses the residual square function to improve the dependence
on $\effhorizon$ while retaining the same decomposition.

\begin{lemma}[Residual square-function bound]
\label{lem:square-leverage}
Under condition~\eqref{eq:oracle-sample-size},
\begin{equation}
 \E\bigl[(R_n^{\rm self})^2\one_{\cE}\bigr]
 \leq\frac{C\effhorizon^7\ell_\subG^5}{n^2\nuMin^2}.
 \label{eq:square-function-remainder-bound}
\end{equation}
\end{lemma}
\par\noindent
See Paragraph~\ref{subsubsec:proof-square-leverage} for the proof.

It remains to control the cross-fitted term.

\begin{lemma}[Leave-one-out remainder]\label{lem:crossfit}
Under condition~\eqref{eq:oracle-sample-size},
\begin{equation}
 \E (R_n^{\rm loo})^2
 \leq\frac{C\effhorizon^7}{n^2\nuMin^2}
      +\frac{C\effhorizon^4}{\nuMin}\Pp(\cE^c).
 \label{eq:crossfit-remainder-bound}
\end{equation}
\end{lemma}
\par\nopagebreak[4]\noindent
See Paragraph~\ref{subsubsec:proof-crossfit} for the proof.

Taking Lemmas~\ref{lem:good}, \ref{lem:crossfit},
\ref{lem:leverage}, and~\ref{lem:square-leverage} as given, combine
inequality~\eqref{eq:crossfit-remainder-bound},
identity~\eqref{eq:collapsed-master-decomposition},
inequalities~\eqref{eq:same-sample-remainder-bound}
and~\eqref{eq:square-function-remainder-bound}.  The deterministic and
square-function controls give, respectively, the two bounds inside the
minimum below.  Since $\effhorizon,\ell_\subG\geq1$, the probability bound
\eqref{eq:good-event-probability} and condition~\eqref{eq:oracle-sample-size}
absorb the complement term in~\eqref{eq:crossfit-remainder-bound} into both
bounds.  Consequently,
\begin{equation}
 \E\bigl[(\widehat\Delta_\subG(\tarstt)-Z_\subG)^2\one_{\cE}\bigr]
 \leq
 C\min\left\{
   \frac{\effhorizon^8\ell_\subG^2}{n^2\nuMin^2},
   \frac{\effhorizon^7\ell_\subG^5}{n^2\nuMin^2}
 \right\}
 =C\mathfrak R_\subG.
 \label{eq:collapsed-good-event-remainder}
\end{equation}

On $\cE$, the left-hand side of
identity~\eqref{eq:collapsed-master-decomposition} equals
$(\widehat a-a)^\top\bar\varepsilon$ by
decomposition~\eqref{eq:subgraph-linearization}.
Thus inequality~\eqref{eq:collapsed-good-event-remainder}, together with
Lemma~\ref{lem:good}, proves Lemma~\ref{lem:bilinear-remainder}.

\paragraph{Proof of Lemma~\ref{lem:leverage}}
\label{subsubsec:proof-leverage}

On $\cE$, the exact decomposition is
\[
 \widehat\Delta_\subG(\tarstt)-Z_\subG
 =\frac1n\sum_i(\widehat a^{(-i)}-a)^\top \varepsilon_i
  +\frac1n\sum_i(\widehat a-\widehat a^{(-i)})^\top \varepsilon_i
 =R_n^{\rm loo}+R_n^{\rm self}.
\]
For this paragraph and Paragraph~\ref{subsubsec:proof-square-leverage},
let $D_N$ and $\mathsf C_\subG$ be the visit-count and transition-count matrices of the
collapsed sample.  Then
$\sum_iL_i=D_N-\mathsf C_\subG$ and $\widehat P_\subG=D_N^{-1}\mathsf C_\subG$.  Put
$\widehat F_\subG=(I-\widehat P_\subG)^{-1}$ and $v_i=\widehat A^{-1}\varepsilon_i$.
Resolvent algebra gives
\begin{equation}
 \widehat a-\widehat a^{(-i)}
 =-\widehat A^{-\top}(L_i/n)^\top\widehat a^{(-i)},\qquad
 R_n^{\rm self}=-\frac1{n^2}\sum_i(\widehat a^{(-i)})^\top L_i v_i.
 \label{eq:same-sample-leverage-identity}
\end{equation}
Let $N_i(s)$ be the number of collapsed visits to $s$ in path $i$,
$N_*=\min_sN_s$, and
\[
 \effhorizon_0=\max_{|I|\leq1}\norm{\widehat F_\subG^{(-I)}\one}_\infty\leq2\effhorizon.
\]
Define the total absolute collapsed residual mass
\[
 B_i=\sum_{t<M_i}|\xi_{i,t}|.
\]
The segment rewards add up over the original path, whereas the value
increments are charged only once per collapsed transition.  Thus
$B_i\leq T_i+2\effhorizon M_i$ holds for every path.  On $\cE$, the additional
bounds $M_i\leq T_i\leq L_0$ give
\begin{equation}
 B_i\leq T_i+2\effhorizon M_i\leq(1+2\effhorizon)L_0.
 \label{eq:joint-path-residual-budget}
\end{equation}
The bound uses the total reward and the number of collapsed transitions
on the same trajectory.  On $\cE$,
\[
 \widehat a_s=\frac{n\widehat F_\subG(\tarstt,s)}{N_s},\qquad
 v_i=n\widehat F_\subG D_N^{-1}\varepsilon_i.
\]
Put
\[
 \tau_i=\sum_s\frac{\widehat F_\subG(\tarstt,s)N_i(s)}{N_s},
 \qquad \widehat f_{\tarstt}=(\widehat F_\subG\one)(\tarstt),
 \qquad \widehat f_*=\norm{\widehat F_\subG\one}_\infty.
\]
Then $\sum_i\tau_i=\widehat f_{\tarstt}$, and the row structure of a path
Laplacian gives
\begin{equation}
 \widehat a^\top|L_i|\one\leq2n\tau_i,
 \qquad
 \norm{v_i}_\infty\leq\frac{n\widehat f_*B_i}{N_*}.
 \label{eq:target-weighted-path-controls}
\end{equation}

The other resolvent identity is
\[
 \widehat a^{(-i)}-\widehat a
 =(\widehat A-L_i/n)^{-\top}(L_i/n)^\top\widehat a.
\]
Elementwise positivity gives
\begin{equation}
 \kappa_i
 :=\frac1n\norm{(\widehat A-L_i/n)^{-1}|L_i|}_{\infty\to\infty}
 \leq\frac{2\effhorizon_0L_0}{N_*-L_0}.
\label{eq:aggregate-leverage-controls}
\end{equation}
Since $\tarstt$ is observed on $\cE$, $\widehat f_{\tarstt}\geq1$; thus
$\pi_{\tarstt}(i)=\tau_i/\widehat f_{\tarstt}$ defines a probability distribution.  The two
resolvent identities give the following shared deterministic bound.

\begin{lemma}[Target-weighted aggregate leverage]
\label{lem:target-weighted-leverage}
On $\cE$,
\[
 |R_n^{\rm self}|
 \leq\frac{2\widehat f_{\tarstt}\widehat f_*}{N_*}
       \E_{\pi_{\tarstt}}\bigl[B_i(1+\kappa_i)\bigr].
\]
\end{lemma}
\par\noindent
See Appendix~\ref{app:proof-target-weighted-leverage} for the proof.

Applying Lemma~\ref{lem:target-weighted-leverage} on $\cE$,
inequality~\eqref{eq:joint-path-residual-budget},
$N_*\geq n\nuMin/2$, and $\widehat f_{\tarstt},\widehat f_*\leq \effhorizon_0\leq2\effhorizon$ hold.
Moreover condition~\eqref{eq:oracle-sample-size} makes
$\max_i\kappa_i\leq1$.  Consequently
\[
 |R_n^{\rm self}|\leq\frac{C\effhorizon^4\ell_\subG}{n\nuMin}.
\]
This target-weighted aggregate calculation, rather than a full operator-norm
Neumann expansion, avoids both an extra $\nuMin^{-1/2}$ loss and a duplicate
use of the maximal path length.
This proves decomposition~\eqref{eq:collapsed-master-decomposition} and
bound~\eqref{eq:same-sample-remainder-bound}, completing the proof of
Lemma~\ref{lem:leverage}.

\paragraph{Proof of Lemma~\ref{lem:square-leverage}}
\label{subsubsec:proof-square-leverage}

Retain $B_i,N_*,\widehat f_{\tarstt},\widehat f_*,\kappa_i$, and $\pi_{\tarstt}$ from
Paragraph~\ref{subsubsec:proof-leverage}.
For trajectory $i$, let
\[
 U_{i,m}=\sum_{t<m}\xi_{i,t},
 \qquad
 [U_i]=\sum_{t<M_i}\xi_{i,t}^2,
 \qquad
 U_i^*=\sup_m|U_{i,m}|,
\]
and extend the process by zero increments after the collapsed path ends.
Successive visits to $\subG\cup\{\termState\}$ are stopping times for the
primitive path.  The stopped Bellman identity and the strong Markov property
therefore make the segment residuals $(\xi_{i,t})_{t<M_i}$ a
martingale-difference sequence for the collapsed filtration.  This uses only
the conditional mean-zero Bellman residual and permits joint
reward--successor dependence.  Moreover, every partial sum is the accumulated
primitive reward between two points of the original path plus one value
difference.  Hence, pathwise,
\begin{equation}
 U_i^*\leq T_i+2\effhorizon.
 \label{eq:collapsed-residual-maximal-bound}
\end{equation}

For every integer $p\geq2$, the discrete reverse Burkholder--Davis--Gundy
inequality, inequality~\eqref{eq:collapsed-residual-maximal-bound}, and
assumption~\eqref{ass:effective-horizon} give
\begin{equation}
 \norm{[U_i]}_{L^p}
 =\norm{[U_i]^{1/2}}_{L^{2p}}^2
 \leq C p^2\norm{U_i^*}_{L^{2p}}^2
 \leq C p^4\effhorizon^2.
 \label{eq:collapsed-square-function-moment}
\end{equation}
The inequality applies to the zero-padded stopped martingale, so no
deterministic bound on $M_i$ is needed at this step.

Since $\nuMin\leq\E_\mu T\leq \effhorizon$,
$n\leq e^{\ell_\subG}$.  Apply Markov's inequality in
\eqref{eq:collapsed-square-function-moment} with
$p=\lceil c\ell_\subG\rceil$, where $c$ is a sufficiently large universal
constant, and take a union bound over the $n$ trajectories.  There is an
event $\mathcal Q$ such that
\begin{equation}
 \Pp(\mathcal Q^c)\leq e^{-26\ell_\subG},
 \qquad
 \max_{i\leq n}[U_i]\leq C\effhorizon^2\ell_\subG^4
 \quad\text{on }\mathcal Q.
 \label{eq:uniform-collapsed-square-function}
\end{equation}
Every collapsed segment contains at least one primitive transition.  Thus
$M_i\leq T_i$, and Cauchy--Schwarz gives
\[
 B_i^2
 \leq M_i\sum_{t<M_i}\xi_{i,t}^2
 \leq T_i[U_i].
\]
On $\cE\cap\mathcal Q$, Part~\ref{item:good-length} of
Lemma~\ref{lem:good} and~\eqref{eq:uniform-collapsed-square-function}
therefore imply
\begin{equation}
 \max_{i\leq n}B_i^2\leq C\effhorizon^3\ell_\subG^5.
 \label{eq:uniform-collapsed-residual-mass}
\end{equation}

Lemma~\ref{lem:target-weighted-leverage}, Jensen's inequality, and
$\max_i\kappa_i\leq1$ give, on $\cE\cap\mathcal Q$,
\[
 (R_n^{\rm self})^2
 \leq
 \frac{C\widehat f_{\tarstt}^2\widehat f_*^2}{N_*^2}
 \E_{\pi_{\tarstt}}B_i^2
 \leq\frac{C\effhorizon^7\ell_\subG^5}{n^2\nuMin^2},
\]
where the last inequality uses $\widehat f_{\tarstt},\widehat f_*\leq2\effhorizon$,
$N_*\geq n\nuMin/2$, and
\eqref{eq:uniform-collapsed-residual-mass}.  On the remaining part
$\cE\cap\mathcal Q^c$, inequality~\eqref{eq:same-sample-remainder-bound}
still gives $(R_n^{\rm self})^2\leq C\effhorizon^8\ell_\subG^2/(n^2\nuMin^2)$.  Consequently,
\[
 \E[(R_n^{\rm self})^2\one_{\cE}]
 \leq\frac{C\effhorizon^7\ell_\subG^5}{n^2\nuMin^2}
      +\frac{C\effhorizon^8\ell_\subG^2}{n^2\nuMin^2}e^{-26\ell_\subG}
 \leq\frac{C\effhorizon^7\ell_\subG^5}{n^2\nuMin^2}.
\]
For the last absorption, condition~\eqref{eq:oracle-sample-size} and
$\nuMin\leq \effhorizon$ imply $e^{\ell_\subG}\geq en\geq C\effhorizon^2\ell_\subG^4$.
This proves~\eqref{eq:square-function-remainder-bound} and completes the
proof of Lemma~\ref{lem:square-leverage}.

\paragraph{Proof of Lemma~\ref{lem:crossfit}}
\label{subsubsec:proof-crossfit}

The diagonal terms in the second moment require a leave-one-out
perturbation bound.  The off-diagonal terms require the smaller change
caused by adding one trajectory to a leave-two-out sample.

\begin{lemma}[Leave-one-out perturbation]\label{lem:leave}
Under condition~\eqref{eq:oracle-sample-size}, for every $i\in[n]$,
\begin{equation}
 \E\norm{\check a^{(-i)}-a}_{\Lambda_\subG^\star}^2
 \leq\frac{C\effhorizon^7}{n\nuMin^2}
 +\frac{C\effhorizon^4}{n^2\nuMin}
 +\frac{C\effhorizon^4}{\nuMin}\Pp(\cE^c).
 \label{eq:stabilized-leave-one-risk}
\end{equation}
\end{lemma}
\par\noindent
See Appendix~\ref{subsec:proof-leave-one} for the proof.

The next lemma controls the common leave-two-out comparison used in each
off-diagonal term.

\begin{lemma}[Leave-two-out perturbation]\label{lem:leave-two}
Under condition~\eqref{eq:oracle-sample-size}, for every distinct $i,j\in[n]$,
\begin{equation}
 \E\norm{\check a^{(-i)}-\check a^{(-i,-j)}}_{\Lambda_\subG^\star}^2
 \leq\frac{C\effhorizon^7}{n^2\nuMin^2}
 +\frac{C\effhorizon^4}{\nuMin}\Pp(\cE^c).
 \label{eq:stabilized-leave-two-difference}
\end{equation}
\end{lemma}
\par\noindent
See Appendix~\ref{subsec:proof-leave-two} for the proof.

Taking Lemmas~\ref{lem:leave} and~\ref{lem:leave-two} as given, fix $i\ne j$ and set
\[
 c=\check a^{(-i,-j)}-a,\qquad
 d_j=\check a^{(-i)}-\check a^{(-i,-j)},\qquad
 d_i=\check a^{(-j)}-\check a^{(-i,-j)}.
\]
In the product
\[
 (c+d_j)^\top \varepsilon_i\,(c+d_i)^\top \varepsilon_j,
\]
the $cc$ term and the two mixed terms vanish by conditioning on trajectory
$i$ or $j$.  Only the $d_jd_i$ term remains.  Cauchy--Schwarz,
identities~\eqref{eq:collapsed-score-moments}, and
Lemmas~\ref{lem:leave} and~\ref{lem:leave-two} give
\[
 \E (R_n^{\rm loo})^2
 \leq\frac1n\E\norm{\check a^{(-i)}-a}_{\Lambda_\subG^\star}^2
     +\E\norm{\check a^{(-i)}-\check a^{(-i,-j)}}_{\Lambda_\subG^\star}^2
 \leq\frac{C\effhorizon^7}{n^2\nuMin^2}
      +\frac{C\effhorizon^4}{\nuMin}\Pp(\cE^c).
\]
Only $\E \varepsilon_i=0$ is used in the cancellation, so joint
reward--successor dependence is allowed.  This proves
inequality~\eqref{eq:crossfit-remainder-bound} and completes
the proof of Lemma~\ref{lem:crossfit}.


\subsection{Proof of Proposition~\ref{prop:bridge}}
\label{subsec:proof-bridge}
\label{subsubsec:proof-whittle}\label{app:conditional-routing}

The comparison averages out the trajectory pairings while preserving the
data used by full TD.  The remaining discrepancy is an orphan contribution
from singular quotient cuts.  We state the exact conditional identity and
the moment bounds needed for the statistical argument, and then complete
the risk comparison before proving its supporting lemmas.

Let
\[
 \mathcal A_\subG=\{N_s>0\text{ for every }s\in\subG\},\qquad
 \widehat\Delta_{\rm TD}(\tarstt)=\widehat V_{\rm TD}(\tarstt)-\Vstar(\tarstt).
\]
Enlarge the probability space by attaching an i.i.d. array of
$\mathrm{Unif}(0,1)$ keys to all transition occurrences, independently of all
trajectories.  At each source state $s$, replace these keys by their ranks
$1,\ldots,N_s$ and then discard both the keys and the original time
indices.  Thus a token in $E_s$ retains only its fresh rank, source,
destination, reward, and residual mark.  Initial stubs retain the trajectory
labels $1,\ldots,n$.  The independently relabeled bag $\cT$ retains the root
states and the arrays $(E_s)_{s\in\cO}$, but no transition token retains
$(i,t)$ or any other information about the original pairing.

The bag determines $(N,\mathsf C,Y)$ and hence the full TD estimator.
It is finer than the collection of empirical transition and reward averages:
the joint reward--successor marks are retained, whereas the trajectory
pairings are not.  This distinction is needed for the conditional assembly
law in Lemma~\ref{lem:fresh-label-assembly}.

Define the Rao--Blackwellized subgraph error and its correction by
\begin{align}
 \widehat\Delta_\subG^{\rm RB}(\tarstt)
 &=\one_{\mathcal A_\subG}
   \E[\widehat V_\subG(\tarstt)-\Vstar(\tarstt)\mid\cT],
 \label{eq:rao-blackwell-term}\\
 \widehat\Delta_\subG^{\rm orp}(\tarstt)
 &=\one_{\mathcal A_\subG}\widehat\Delta_{\rm TD}(\tarstt)
   -\widehat\Delta_\subG^{\rm RB}(\tarstt).
 \label{eq:orphan-definition}
\end{align}
Both variables vanish on $\mathcal A_\subG^c$.  Thus the exact decomposition is
\begin{equation}
 \one_{\mathcal A_\subG}\widehat\Delta_{\rm TD}(\tarstt)
 =\widehat\Delta_\subG^{\rm RB}(\tarstt)
  +\widehat\Delta_\subG^{\rm orp}(\tarstt).
 \label{eq:full-error-whittle-decomposition}
\end{equation}
The substantive identity is the singular-cut representation
\eqref{eq:orphan-term}, proved in Lemma~\ref{lem:whittle}; it explains the
name \emph{orphan} and makes the correction small.  Full TD need not itself
equal the conditional expectation in~\eqref{eq:rao-blackwell-term}.
Since $\mathcal A_\subG$ is $\cT$-measurable, conditional Jensen gives
\begin{equation}
 \E[\widehat\Delta_\subG^{\rm RB}(\tarstt)^2]
 \leq\E[\{\widehat V_\subG(\tarstt)-\Vstar(\tarstt)\}^2].
 \label{eq:rao-blackwell-jensen}
\end{equation}

For $\subG=\{\tarstt\}$, the subgraph Bellman equation telescopes to
first-visit Monte Carlo: on $\{N_{\tarstt}>0\}$, it averages the returns
following the first visit over trajectories that hit $\tarstt$.
Rao--Blackwellization of this estimator is studied by
\citet[Theorem~2 and equation~(23)]{grunewalder2011optimal}.
Their minimum-variance result uses a sufficient complete statistic,
assumptions on reward estimation, and unbiasedness conditional on visiting
the state.  We do not invoke that optimality result or identify their
statistic with the finer marked bag $\cT$.
Their discussion also includes exponentially decaying coverage bias.
The task here is to prove an exact correction to the conditional average
for general $\subG$, and to control its absolute first moment under an
arbitrary bounded joint reward--successor kernel.

We represent observed transitions as distinguishable edges of a directed
graph. Lemma~\ref{lem:graph-determinant} gives a determinant identity for
joining edges at selected vertices; both the counting formulas and the
identity with residual sums follow from it. These yield the conditional
identity and orphan bounds used in Proposition~\ref{prop:bridge}.

Work first on
$\mathcal A_\subG=\{N_s>0\text{ for every }s\in\subG\}$, and write
$\cO=\{s:N_s>0\}$ and $R=\cO\setminus\subG$.
Represent each transition token in the relabeled data $\cT$ by a distinct
edge $e:s\to u$, retaining its fresh label, reward $R_e$, and residual
weight $w_e=R_e+\Vstar(u)-\Vstar(s)$. Parallel edges remain distinct
even when their recorded values agree; original trajectory and time
indices are absent, as in Subsection~\ref{subsec:proof-bridge}.

Represent the terminal state by $q$, with $\Vstar(q)=0$. For each
trajectory label $i$, add a distinct zero-weight edge from $q$ to its
initial state, a loop at $q$ if that state is terminal. These edges
encode the labeled initial stubs; they are not extra observations.
The resulting graph has vertex set $\cO\cup\{q\}$ and equal indegree and
outdegree at every vertex.  Indeed, at $s\in\cO$,
\begin{equation}
 N_s=\kappa_s+\sum_{u\in\cO}\mathsf C_{us},
 \qquad
 \kappa_s=\#\{i:S_0^{(i)}=s\},
 \label{eq:deterministic-bag-balance}
\end{equation}
and at $q$ both degrees equal $n$.  A loop is counted once in each degree.
The matrices of interest omit the row and column of $q$:
\[
 L=D_N-\mathsf C,
 \qquad D_N=\diag(N_s:s\in\cO),
 \qquad b_s=\sum_{e\text{ leaving }s}w_e
           =\sum_{i,t:S_t^{(i)}=s}\delta_{i,t}.
\]
In particular, terminal transitions are included in $N_s$.
The full error vector
$\widehat\Delta_{\rm TD}=\widehat V_{\rm TD}-\Vstar|_{\cO}$ satisfies
$L\widehat\Delta_{\rm TD}=b$.  Every observed state has an observed path
to $q$, so $L$ is invertible and $\det L>0$.

For a vertex $v$, let $E^-(v)$ and $E^+(v)$ be the sets of edges entering
and leaving $v$.  A choice of successors at $v$ is a bijection
$E^-(v)\to E^+(v)$: after arriving along an edge, take the outgoing edge
assigned to it.  No successor is chosen at $q$.
A \emph{full assembly} specifies these bijections at all vertices of
$\cO$, with every successor sequence eventually reaching $q$. Its edges
split into $n$ trails from $q$ to $q$, one starting with each added edge;
removing those added edges recovers the labeled complete trajectories.
A trail may revisit vertices but cannot repeat an edge.
Write $\mathfrak A(\cT)$ for the set of full assemblies and $A_{\rm act}$
for the one induced by the relabeled observations.

A \emph{raw $\subG$-cut} is a family $c=(c_r)_{r\in R}$ of bijections
$c_r:E^-(r)\to E^+(r)$ such that every successor sequence stops at
$\subG\cup\{q\}$. This excludes successor cycles inside $R$, although
the original graph may have cycles. Denote the set of cuts by
$\mathfrak C_R(\cT)$; pairings at $\subG$ are not specified.

Given $c$, follow each edge leaving $\subG\cup\{q\}$ until its first
subsequent arrival in that set. Replace the trail by an edge with the same
endpoints, summed weight, and ordered list of original edges as its
identity, so the replacement is reversible. Degree balance and the
degrees $N_s$ on $\subG$ are preserved, as proved in
Lemma~\ref{lem:graph-determinant}. For $s,u\in\subG$, set
\[
 (\mathsf C_c)_{su}=\#\{\text{new edges from }s\text{ to }u\},
 \qquad
 (b_c)_s=\sum_{e\text{ a new edge leaving }s}w_e,
 \qquad
 L_c=D_{N_\subG}-\mathsf C_c.
\]
The subscript $c$ records the bijections, not a submatrix.
Trails beginning at $q$ encode prefixes before the first visit to $\subG$;
their weights are omitted from $b_c$. A cut may have no full completion,
so $L_c$ may be singular. Different cuts are counted separately even
when they produce the same matrices.

We use the adjugate convention
$(\adj A)_{ji}=(-1)^{i+j}\det A^{\langle i,j\rangle}$, where
$A^{\langle i,j\rangle}$ deletes row $i$ and column $j$.
This definition also applies to singular matrices.
The determinant of a $0\times0$ matrix is one, empty products are one,
empty sums are zero, and $0!=1$.
The vector $e_s$ is the coordinate vector for vertex $s$.

\begin{lemma}[Exact conditional routing identity]\label{lem:whittle}
On $\mathcal A_\subG$, the conditional law of the actual cut is
\begin{equation}
 \Pp_{\rm act}(c\mid\cT)
 =\frac{\det L_c}{\det L\prod_{r\in R}(N_r-1)!}.
 \label{eq:whittle-cut-density}
\end{equation}
Moreover,
\begin{align}
 \widehat\Delta_{\rm TD}(\tarstt)
 &=\frac{\displaystyle\sum_{c\in\mathfrak C_R(\cT)}
       e_{\tarstt}^\top\adj(L_c)b_c}
       {\det L\prod_{r\in R}(N_r-1)!},
 \label{eq:whittle-marked-adjugate-mean}\\
 \widehat\Delta_\subG^{\rm RB}(\tarstt)
 &=\frac{\displaystyle\sum_{\substack{c\in\mathfrak C_R(\cT)\\\det L_c>0}}
       e_{\tarstt}^\top\adj(L_c)b_c}
       {\det L\prod_{r\in R}(N_r-1)!}.
 \label{eq:rb-nonsingular-cut-sum}
\end{align}
Consequently, the correction in~\eqref{eq:orphan-definition} is
\begin{equation}
 \widehat\Delta_\subG^{\rm orp}(\tarstt)
 =\frac{\displaystyle\sum_{\substack{c\in\mathfrak C_R(\cT)\\\det L_c=0}}
       e_{\tarstt}^\top\adj(L_c)b_c}
       {\det L\prod_{r\in R}(N_r-1)!}.
 \label{eq:orphan-term}
\end{equation}
\end{lemma}
\nopagebreak
\par\noindent
See Subsubsection~\ref{subsec:proof-conditional-identity} for the proof.

The factor $\det L_c$ cancels the denominator in Cramer's rule.
Singular cuts have zero conditional probability, but their adjugate
numerators can be nonzero; these form the orphan contribution.

The next lemma bounds the correction by the probability that all rooted
paths miss the target.  Its proof uses the exact conditional identity,
followed by singular-cut reduction and Palm reassembly.

\begin{lemma}[Orphan first-moment bound]\label{lem:orphan}
Let $\rho_{\tarstt}=\E_{\tarstt}\sum_{t<T}|\delta_t|\leq3\effhorizon^2$.
For the correction in~\eqref{eq:orphan-definition},
\begin{equation}
 \E|\widehat\Delta_\subG^{\rm orp}(\tarstt)|\leq(1-p_{\tarstt})^n\rho_{\tarstt}.
 \label{eq:palm-orphan-bound}
\end{equation}
\end{lemma}
\par\noindent
See Subsubsection~\ref{subsubsec:proof-orphan} for the proof.

Interpolation and the effective-horizon bounds convert this estimate to the
two second moments needed for the risk comparison.

\begin{lemma}[Orphan and coverage second moments]\label{lem:bridge-remainders}
Under condition~\eqref{eq:oracle-sample-size},
\begin{align}
 \E[\widehat\Delta_\subG^{\rm orp}(\tarstt)^2]
 &\leq C\mathfrak R_\subG,
 \label{eq:orphan-second-moment}\\
 \E[\widehat\Delta_{\rm TD}(\tarstt)^2\one_{\mathcal A_\subG^c}]
 &\leq C\mathfrak R_\subG.
 \label{eq:coverage-failure-second-moment}
\end{align}
\end{lemma}
\par\noindent
See Appendix~\ref{subsec:proof-bridge-remainders} for the proof.

Taking these lemmas as given, apply Young's inequality to
\eqref{eq:full-error-whittle-decomposition} and add the coverage complement:
\[
 \E[\widehat\Delta_{\rm TD}(\tarstt)^2]
 \leq(1+\eta)\E[\widehat\Delta_\subG^{\rm RB}(\tarstt)^2]
 +(1+\eta^{-1})\E[\widehat\Delta_\subG^{\rm orp}(\tarstt)^2]
 +\E[\widehat\Delta_{\rm TD}(\tarstt)^2\one_{\mathcal A_\subG^c}].
\]
Inequality~\eqref{eq:rao-blackwell-jensen} and
Lemma~\ref{lem:bridge-remainders} prove
\eqref{eq:full-to-collapsed-risk}, completing the proof of
Proposition~\ref{prop:bridge}.

\subsubsection{Proof of Lemma~\ref{lem:whittle}}
\label{subsec:proof-conditional-identity}

For the graph associated with a bag $\tau$, the families $\phi\in\Phi_R$
of Lemma~\ref{lem:graph-determinant} are exactly the cuts
$c\in\mathfrak C_R(\tau)$ of
Subsection~\ref{subsubsec:proof-whittle}.  Their matrices satisfy
$L_\phi=L_c$ and $b_\phi=b_c$.  In particular, edges starting at $q$ never
contribute to a retained row.  We first identify the conditional law of
all the bijections, and then specialize the graph identity to the two
sums needed in Lemma~\ref{lem:whittle}.

The next lemma applies to arbitrary finite graphs with real edge weights.
It allows disconnected graphs, since joining edges can create a component
that does not contain $q$.

\begin{lemma}[Determinant identity for a directed multigraph]
\label{lem:graph-determinant}
Let $D$ be a finite directed multigraph with vertex set $\cO\cup\{q\}$,
where $q\notin \cO$. Fix an ordering of $\cO$ and use the induced order
on each subset unless another order is specified.
Loops and parallel edges are allowed, and all edges
are distinguishable. Every vertex has equal indegree and outdegree.
Write $E^-(v)$ and $E^+(v)$ for the incoming and outgoing edge sets,
with $N_v=|E^-(v)|=|E^+(v)|\geq1$ for $v\in\cO$.
Loops count once in each degree. Assign arbitrary real weights $w_e$,
which do not affect degrees or edge counts.
For $u,v\in \cO$, define
\[
 \mathsf C_{uv}=\#\{\text{edges from $u$ to $v$}\},\qquad
 b_v=\sum_{e\in E^+(v)}w_e,\qquad
 L=\diag(N_v:v\in \cO)-\mathsf C.
\]
Here $\mathsf C=(\mathsf C_{uv})_{u,v\in \cO}$ and
$b=(b_v)_{v\in \cO}$ is a column vector.
Edges to $q$ are included in $N_v$, but the row
and column of $q$ are omitted from $L$.

Choose $R\subseteq\cO$ and put $\subG=\cO\setminus R$. Let $\Phi_R$
consist of families of bijections $\phi_r:E^-(r)\to E^+(r)$,
$r\in R$, whose prescribed successors, starting from any edge,
eventually leave $R$. Equivalently, they have no successor cycle
inside $R$. Vertices may be revisited with different edges; the
original graph need not be acyclic.

For $\phi\in\Phi_R$, apply the trail replacement of
Subsection~\ref{subsubsec:proof-whittle}, retaining each trail's endpoints,
summed weight, and ordered edge list. Define $\mathsf C_\phi,b_\phi,L_\phi$
as there for the new graph on $\subG\cup\{q\}$; its degrees on $\subG$
remain $N_v$.
The subscript $\phi$ specifies the bijections, not a submatrix;
different families producing the same matrices are counted separately.

To include the weight sums, append a final column labeled
$\star\notin\cO\cup\{q\}$:
\[
 M=[\,L\mid b\,],\qquad M_\phi=[\,L_\phi\mid b_\phi\,].
\]
The label $\star$ is not a graph vertex; the other row and column
labels remain those of $L$ and $L_\phi$. Determinants below refer
only to square submatrices.
For ordered index lists $I=(i_1,\ldots,i_k)$ and $J=(j_1,\ldots,j_k)$,
write $A_{I,J}=(A_{i_a,j_b})_{a,b=1}^k$.
Subsets carry their induced order, and single indices select one row or
column. Fix an order $R=(r_1,\ldots,r_m)$. Let $I$ be an ordered list of distinct elements
of $\subG$ and $J$ an equally long ordered list of distinct elements of
$\subG\cup\{\star\}$. The notation $(R,I)$ means $R$ followed by $I$,
and likewise for $(R,J)$. Then
\begin{equation}
 \displaystyle
 \sum_{\phi\in\Phi_R}\det\bigl((M_\phi)_{I,J}\bigr)
 =\left(\prod_{r\in R}(N_r-1)!\right)
    \det\bigl(M_{(R,I),(R,J)}\bigr).
 \label{eq:graph-determinant-identity}
\end{equation}
An empty determinant and an empty product are both one; an empty
sum is zero; and $0!=1$. In particular, $I,J,\subG$, or $R$ may be
empty. No connectivity or matrix-invertibility assumption is made.
\end{lemma}
\par\noindent
See Paragraph~\ref{subsec:graph-determinant} for the proof.

The next lemma identifies the conditional distribution of the full assembly.

Let $Q_s$ be the joint reward--successor law at $s$. A relabeled bag
$\tau$ consists of the labeled initial states $s_i^0$ and arrays
$(r_{s,k},u_{s,k})_{k=1}^{N_s}$ indexed by source and fresh label.
Residual weights are functions of these coordinates.
For a nonnegative integer count vector $\bm N=(N_s)_{s\in\cS}$, define
\begin{equation}
 \Xi_{\bm N}(d\tau)
 :=\left\{\prod_{s\in\cS}\frac1{N_s!}\right\}
   \left\{\prod_{i=1}^n\mu(ds_i^0)\right\}
   \left\{\prod_{s\in\cS}\prod_{k=1}^{N_s}
                         Q_s(dr_{s,k},du_{s,k})\right\}.
 \label{eq:fresh-label-reference-measure}
\end{equation}
This product measure includes data that cannot be assembled into complete
trajectories.  Define $\mathfrak A(\tau)=\varnothing$ for such data, in
particular when the degree balance fails or an initial state or successor
has no outgoing edge despite being nonterminal.

\begin{lemma}[Fresh-label assembly law]\label{lem:fresh-label-assembly}
For every nonnegative measurable function $\Phi$ of the relabeled data and
a compatible full assembly,
\begin{equation}
 \E\Phi(\cT,A_{\rm act})
 =\sum_{\bm N}\int
   \sum_{a\in\mathfrak A(\tau)}\Phi(\tau,a)\,
   \Xi_{\bm N}(d\tau).
 \label{eq:fresh-label-coordinate-identity}
\end{equation}
Consequently, conditional on $\cT=\tau$, the actual full assembly is
uniform on $\mathfrak A(\tau)$ for $\Pp_\cT$-almost every $\tau$.
\end{lemma}
\par\noindent
See Paragraph~\ref{subsec:proof-fresh-labels} for the proof.

The following counting formulas determine the conditional probabilities of cuts.

\begin{lemma}[Cycle-free pairing counts]\label{lem:assembly-counts}
For every graph satisfying the hypotheses of
Lemma~\ref{lem:graph-determinant} and every $A\subseteq\cO$,
\begin{equation}
 \#\{\text{bijections at $A$ whose successors eventually leave $A$}\}
 =\det(L_{AA})\prod_{s\in A}(N_s-1)!.
 \label{eq:boundary-root-pairing-count}
\end{equation}
In particular, let $\tau$ have $N_s\geq1$ outgoing edges at each
$s\in\cO$, all initial states and destinations in
$\cO\cup\{\termState\}$, and degree balance
\eqref{eq:deterministic-bag-balance}.  The numbers of full assemblies and
of full completions of a fixed cut $c$ are
\begin{align}
 \Omega_{\rm full}(\tau):=|\mathfrak A(\tau)|
 &=\det L\prod_{s\in\cO}(N_s-1)!,
 \label{eq:full-assembly-count}\\
 \#\{\text{full completions of }c\}
 &=\det L_c\prod_{s\in\subG}(N_s-1)!.
 \label{eq:cut-completion-count}
\end{align}
These formulas include graphs with no full assembly.  All their displayed
determinants are nonnegative.
\end{lemma}
\par\noindent
See Paragraph~\ref{subsec:proof-assembly-counts} for the proof.

The residual sums are handled by the following adjugate identity.

\begin{lemma}[Marked adjugate sum]\label{lem:marked-adjugate}
Let $\tau$ consist of distinguishable transition edges with sources in
$\cO$ and destinations in $\cO\cup\{\termState\}$, and labeled initial
states in that same set.  Assume positive outgoing degrees on $\cO$,
\eqref{eq:deterministic-bag-balance}, and arbitrary real transition
weights.  For every nonempty
$\subG\subseteq\cO$, every $g\in\subG$, and $R=\cO\setminus\subG$,
\begin{equation}
 \sum_{c\in\mathfrak C_R(\tau)}e_g^\top\adj(L_c)b_c
 =\left\{\prod_{r\in R}(N_r-1)!\right\}e_g^\top\adj(L)b.
 \label{eq:marked-adjugate-sum}
\end{equation}
No matrix in this identity is required to be invertible.
\end{lemma}
\par\noindent
See Paragraph~\ref{subsec:proof-marked-routing} for the proof.

Taking Lemmas~\ref{lem:fresh-label-assembly}, \ref{lem:assembly-counts},
and~\ref{lem:marked-adjugate} as given, we now prove
Lemma~\ref{lem:whittle}.

Return to the observed data on $\mathcal A_\subG$.
Lemma~\ref{lem:fresh-label-assembly} makes all compatible full assemblies
equally likely conditional on $\cT$.  A full assembly restricts to exactly
one cut, and the number of assemblies with that restriction is its number
of full completions.  Dividing~\eqref{eq:cut-completion-count} by
\eqref{eq:full-assembly-count} and canceling the factorials at $\subG$
gives~\eqref{eq:whittle-cut-density}.  In particular, singular cuts have
zero conditional probability, while every cut with positive determinant
has a full completion.

Every completion of a fixed nonsingular cut has the same empirical
collapsed rows.  On an edge obtained from a trail
$v_0\to v_1\to\cdots\to v_k$ starting in $\subG$, its residual weight is
\[
 \sum_{j=0}^{k-1}
   \{R_j+\Vstar(v_{j+1})-\Vstar(v_j)\}
 =\sum_{j=0}^{k-1}R_j+\Vstar(v_k)-\Vstar(v_0).
\]
Thus the collapsed Bellman error vector is $L_c^{-1}b_c$, with the same
terminal-inclusive denominators as the estimator in the main text.
The parts of trajectories before their first visit to $\subG$ start at
$q$ and are correctly omitted from these rows.  Taking the conditional
average over cuts, and using
$(\det L_c)L_c^{-1}=\adj(L_c)$ on the nonsingular cuts, proves
\eqref{eq:rb-nonsingular-cut-sum}.

On the other hand, Lemma~\ref{lem:marked-adjugate} with $g=\tarstt$,
together with
\[
 \frac{e_{\tarstt}^\top\adj(L)b}{\det L}
   =e_{\tarstt}^\top L^{-1}b
   =\widehat\Delta_{\rm TD}(\tarstt),
\]
proves~\eqref{eq:whittle-marked-adjugate-mean}.  Subtracting the two sums
proves~\eqref{eq:orphan-term}, by
\eqref{eq:orphan-definition}.  All these conditional identities are on
$\mathcal A_\subG$; the correction is zero on its complement by definition.
This proves Lemma~\ref{lem:whittle}.

\paragraph{Proof of Lemma~\ref{lem:graph-determinant}}
\label{subsec:graph-determinant}

The trails partition the original edges. An edge starting in $R$ has a
unique predecessor under the inverse bijection. Tracing predecessors
must reach $\subG\cup\{q\}$, where none is prescribed; otherwise
finiteness gives a forbidden successor cycle. Successor tracing also
terminates, proving coverage and uniqueness. Each original edge entering
or leaving a remaining vertex is therefore the last or first edge of
exactly one new edge. All remaining indegrees and outdegrees are preserved,
so the new graph satisfies the degree hypotheses even if disconnected.

\medskip
\noindent\textbf{1. An average over permutations.}
For integers $k\geq0$ and $d\geq1$, let $A$ be a real $k\times k$
matrix, and let $x_1,\ldots,x_d,y_1,\ldots,y_d$ be vectors in
$\R^k$.
With the sum taken over all $d!$ permutations $\sigma$ of
$\{1,\ldots,d\}$, we claim
\begin{equation}
 \frac1{d!}\sum_\sigma
 \det\!\left(A+\sum_{\alpha=1}^d
                    x_\alpha y_{\sigma(\alpha)}^{\top}\right)
 =\det\!\left(A+\frac1d
           \left(\sum_{\alpha=1}^d x_\alpha\right)
           \left(\sum_{\beta=1}^d y_\beta\right)^{\top}\right).
 \label{eq:graph-permutation-average}
\end{equation}

Expand the determinant on the left by multilinearity in its columns.
A term chooses a set $Q$ of column positions from the added matrix
and uses columns of $A$ elsewhere. For each $j\in Q$, it then
chooses an index $\alpha_j$, giving column $x_{\alpha_j}$ with
coefficient $(y_{\sigma(\alpha_j)})_j$. Repeated indices give
repeated columns, so only distinct $\alpha_j$ can contribute.
For distinct indices, the average of
$\prod_{j\in Q}(y_{\sigma(\alpha_j)})_j$
is independent of the chosen indices: their images under a uniform
permutation run equally often over all ordered lists of distinct
indices. If $|Q|\geq2$, exchange the indices assigned to two fixed
positions in $Q$. The averaged coefficient is unchanged and the
determinant changes sign. These terms therefore cancel in pairs.

The term with $Q=\varnothing$ is $\det A$. When $Q=\{j\}$, the
mean added column is
\[
 \frac1d\left(\sum_{\alpha=1}^d x_\alpha\right)
           \left(\sum_{\beta=1}^d (y_\beta)_j\right).
\]
These are precisely the added columns on the right of
\eqref{eq:graph-permutation-average}. All those columns are proportional, so
terms using two or more of them are zero. The two expansions agree.
This proof also applies when $A$ is singular; $k=0$ is immediate.
By restricting the vectors to selected rows and columns, the same
argument applies to any square minor of a rectangular matrix update.

\medskip
\noindent\textbf{2. Remove one vertex.}
Consider a graph with vertex set $H\cup\{r,q\}$, where the union is
disjoint. Throughout this step, $L,b,M,N$ refer to this current
graph. It may already have been obtained by replacing trails in
another graph. We prove
\begin{equation}
 \sum_{\phi\in\Phi_{\{r\}}}
       \det\bigl((M_\phi)_{I,J}\bigr)
   =(N_r-1)!\det\bigl(M_{(r,I),(r,J)}\bigr),
 \label{eq:graph-single-vertex}
\end{equation}
for equal-length ordered lists $I$ in $H$ and $J$ in
$H\cup\{\star\}$.

Let $a$ be the number of loops at $r$, and put $d=N_r-a=L_{rr}$.
There are $d$ non-loop edges entering $r$ and $d$ non-loop edges
leaving $r$, because its indegree and outdegree agree.
If $d=0$, every incident edge is a loop. Since $N_r>0$, any bijection
between the incoming and outgoing edges contains a closed sequence,
so $\Phi_{\{r\}}$ is empty. Also $L_{rr}=0$ and $L_{Hr}=0$,
where $L_{Hr}$ means the column with row indices $H$ and column
index $r$. Thus the column labeled $r$ on the right of
\eqref{eq:graph-single-vertex} is zero. Both sides are zero in this case.

Suppose $d\geq1$. A permitted bijection first distributes the distinct
loops into $d$ ordered, possibly empty lists, one before each non-loop
outgoing edge, then matches the non-loop incoming edges to these lists.
This description is unique: each incoming edge follows its list and
leaves $r$, while any unused loop would lie in a forbidden closed sequence.

The number of choices for the lists is
\begin{equation}
 a!\binom{a+d-1}{d-1}
   =\frac{(N_r-1)!}{(d-1)!}.
 \label{eq:graph-loop-lists}
\end{equation}
Indeed, order the $a$ loops, then divide that order into $d$ lists;
$\binom{a+d-1}{d-1}$ choices of $d-1$ separators allow empty lists.
Each choice of lists has $d!$ matchings. There are therefore
$d(N_r-1)!$ permitted bijections at $r$.

Fix the lists. Index the incoming non-loop edges by
$\alpha=1,\ldots,d$, and let their starting vertices be $s_\alpha$.
Index the lists by $\beta=1,\ldots,d$; let $t_\beta$ be the endpoint
of the outgoing edge after list $\beta$, and let $z_\beta$ be the
sum of the weights on that list and that outgoing edge.
Both $s_\alpha$ and $t_\beta$ belong to $H\cup\{q\}$.
For $v\in H$, let $e_v$ be the coordinate vector in $\R^H$
and set $e_q=0$.
Define
\[
 x_\alpha=e_{s_\alpha},\qquad
 y_\beta^{\top}=[\,-e_{t_\beta}^{\top}\mid z_\beta\,].
\]
Thus $y_\beta$ has coordinates labeled by $H\cup\{\star\}$.
Counting incoming and outgoing edges gives
\begin{equation}
 \sum_{\alpha=1}^d x_\alpha=-L_{Hr},\qquad
 \sum_{\beta=1}^d y_\beta^{\top}=[\,L_{rH}\mid b_r\,].
 \label{eq:graph-vector-sums}
\end{equation}
The last coordinate is $b_r$ because every edge leaving $r$ occurs
in exactly one of the chosen lists or as its last outgoing edge.

Match incoming edge $\alpha$ to list $\sigma(\alpha)$ and its
outgoing edge. The resulting matrix is exactly
\begin{equation}
 M_\sigma=[\,L_{HH}\mid b_H\,]
       +\sum_{\alpha=1}^d
                x_\alpha y_{\sigma(\alpha)}^{\top}.
 \label{eq:graph-matrix-update}
\end{equation}
Indeed, replacing $s\to r$ by $s\to t_{\sigma(\alpha)}$ adds the
negative adjacency entry without changing the outdegree of $s\in H$.
The incoming weight is already in $b_H$, so only $z_{\sigma(\alpha)}$
is added. Edges starting at $q$ change no row, since $e_q=0$.

Apply \eqref{eq:graph-permutation-average} to the selected minor of
\eqref{eq:graph-matrix-update}, using \eqref{eq:graph-vector-sums}. It yields
\[
 \frac1{d!}\sum_\sigma\det\bigl((M_\sigma)_{I,J}\bigr)
       =\det(\overline M_{I,J}),\qquad
 \overline M=[\,L_{HH}\mid b_H\,]
                   -\frac1d L_{Hr}[\,L_{rH}\mid b_r\,].
\]
This expression is independent of the chosen lists. On the other hand,
\[
 M_{(r,I),(r,J)}
 =\begin{pmatrix}
     d & M_{r,J}\\
     L_{I,r} & M_{I,J}
   \end{pmatrix}.
\]
Subtracting $L_{ir}/d$ times the first row from the row indexed by
each $i\in I$ makes the first column zero below its first entry.
Expanding along that column gives
\[
 \det\bigl(M_{(r,I),(r,J)}\bigr)
     =d\det(\overline M_{I,J}).
\]
Multiplying the average by the $d!$ matchings and the number of
list choices in \eqref{eq:graph-loop-lists} proves \eqref{eq:graph-single-vertex}.
All row operations allow $J$ to include the weight-sum column.

\medskip
\noindent\textbf{3. Remove the vertices of $R$ one at a time.}
We induct on $m=|R|$. If $R$ is empty, there is only the empty
family of bijections; no edges are joined, and \eqref{eq:graph-determinant-identity} is
immediate.

For $m\geq1$, write $R=(r_1,\ldots,r_m)$ and $R'=(r_1,\ldots,r_{m-1})$.
A family in $\Phi_R$ corresponds to a choice $\psi\in\Phi_{R'}$
followed by a permitted bijection at $r_m$ after replacing the
$\psi$-trails. The last and first original edges identify the incoming
and outgoing incidences at $r_m$, making this correspondence bijective.
A successor cycle inside $R$ either avoids $r_m$ and violates $\psi$,
or contracts to a cycle at $r_m$; conversely, the latter expands to
a cycle inside $R$. Thus termination is preserved in both directions,
including when no final choice exists. The two-stage replacement
preserves the final ordered edge lists and summed weights.

The remaining degrees are unchanged, as proved at the beginning.
Apply \eqref{eq:graph-single-vertex} to $r_m$ and then the induction hypothesis to
$R'$ with row list $(r_m,I)$ and column list $(r_m,J)$:
\begin{align*}
 \sum_{\phi\in\Phi_R}\det\bigl((M_\phi)_{I,J}\bigr)
 &= (N_{r_m}-1)!
    \sum_{\psi\in\Phi_{R'}}
       \det\bigl((M_\psi)_{(r_m,I),(r_m,J)}\bigr)\\
 &= (N_{r_m}-1)!\left(\prod_{r\in R'}(N_r-1)!\right)
       \det\bigl(M_{(R,I),(R,J)}\bigr).
\end{align*}
The row and column orders are exactly the displayed orders, so no
additional sign occurs. This proves \eqref{eq:graph-determinant-identity}
and Lemma~\ref{lem:graph-determinant}.

\paragraph{Proof of Lemma~\ref{lem:fresh-label-assembly}}
\label{subsec:proof-fresh-labels}

Fix the $n$ ordered complete trajectories.  Conditional on those
trajectories, the fresh ranks at source $s$ give each of its $N_s!$
label assignments probability $1/N_s!$, independently across sources.
The labeled trajectories determine a pair $(\tau,a)$ with
$a\in\mathfrak A(\tau)$: $\tau$ forgets the consecutive-edge relations,
and $a$ records them as bijections.  Conversely, starting from the added
edge labeled $i$ and following the bijections recovers trajectory $i$
with all its assigned edge labels.  This is a bijection of descriptions.

On each fixed choice of lengths, state sequences, and labels, the
trajectory law is the product of the initial-state measures and the joint
kernels $Q_s$ along its transitions.  The change of description merely
reorders these factors by source and fresh label.  It preserves each
reward together with its successor, so it does not require independence
between these two coordinates.  Multiplying by the label probability,
then summing over all discrete choices and integrating the reward
coordinates, gives~\eqref{eq:fresh-label-coordinate-identity}.  Nonnegative
integrands justify the rearrangement.  For fixed $\tau$, every compatible
$a$ has the same measure factor $\Xi_{\bm N}(d\tau)$, which proves the
conditional uniformity and Lemma~\ref{lem:fresh-label-assembly}.

\paragraph{Proof of Lemma~\ref{lem:assembly-counts}}
\label{subsec:proof-assembly-counts}

Apply Lemma~\ref{lem:graph-determinant} with removed set $A$ and empty
lists $I,J$. Each summand is one, and the right-hand determinant is
$\det L_{AA}$, proving~\eqref{eq:boundary-root-pairing-count}
without connectivity or invertibility. Taking $A=\cO$ proves
\eqref{eq:full-assembly-count}.

For a fixed cut $c$, replace its trails by distinct edges. Full completions
correspond to bijections at $\subG$ in the new graph with successors
reaching $q$: the last and first original edges identify the incidences,
allowing unique expansion. Termination is preserved in both directions
as in Step~3 of the graph lemma. Apply the first formula to this balanced
graph, with unchanged degrees $N_s$ and $A=\subG$, to obtain
\eqref{eq:cut-completion-count}. Each determinant is a nonnegative count
divided by positive factorials, proving Lemma~\ref{lem:assembly-counts}.

\paragraph{Proof of Lemma~\ref{lem:marked-adjugate}}
\label{subsec:proof-marked-routing}

Order $\subG=(g_1,\ldots,g_\ell)$, write $g=g_j$, and in
Lemma~\ref{lem:graph-determinant} take
\[
 I=(g_1,\ldots,g_\ell),\qquad
 J=(g_1,\ldots,g_{j-1},\star,g_{j+1},\ldots,g_\ell).
\]
By the definition of $M_c$, its selected matrix is $L_c$ with
the column labeled $g$ replaced by $b_c$.  Expanding along that column
gives determinant $e_g^\top\adj(L_c)b_c$.
On the right, the row order is $(R,\subG)$ and the column order is the
same list with $g$ replaced by $\star$.  Thus the selected matrix is $L$
with the column labeled $g$ replaced by $b$, after a common permutation
of its rows and columns.  A common permutation contributes the square
of its sign, which is one.  Its determinant is therefore
$e_g^\top\adj(L)b$.  The graph identity proves
\eqref{eq:marked-adjugate-sum} and Lemma~\ref{lem:marked-adjugate}.

\subsubsection{Proof of Lemma~\ref{lem:orphan}}
\label{subsubsec:proof-orphan}

The singular sum in~\eqref{eq:orphan-term} can be rewritten using complete
rooted paths that avoid $\tarstt$, together with excursions that start at
$\tarstt$ and return to it.  For a trajectory started at $\tarstt$, let
\[
 T_{\{\tarstt\}}=\inf\{t\geq1:S_t\in\{\tarstt,\termState\}\},
 \qquad
 A=\sum_{t<T_{\{\tarstt\}}}|\delta_t|,
 \qquad
 J=\one\{S_{T_{\{\tarstt\}}}=\termState\}.
\]
Thus $J=0$ means the excursion returns to $\tarstt$, and $J=1$ means it
terminates.  Write $\alpha=\Pp_{\tarstt}(J=1)$.

\begin{lemma}[Orphan reassembly]\label{lem:palm-reassembly}
Let $(A_j,J_j)_{j\geq1}$ be independent copies of $(A,J)$ under
$\Pp_{\tarstt}$.  Then
\begin{equation}
 \E|\widehat\Delta_\subG^{\rm orp}(\tarstt)|
 \leq(1-p_{\tarstt})^n
 \sum_{k\geq1}\E_{\tarstt}\!\left[
  \one\{J_1=\cdots=J_k=0\}\frac1k\sum_{j=1}^k A_j
 \right].
 \label{eq:orphan-palm-reassembly}
\end{equation}
\end{lemma}
\par\noindent
See Paragraph~\ref{subsec:proof-palm-reassembly} for the proof.

Taking the lemma as given, the effective-horizon assumption implies
$\alpha>0$.  Otherwise, each excursion would return almost surely, and
the strong Markov property would force infinitely many returns before
termination, contradicting the almost-sure finiteness of $T$.
Independence and symmetry of the $k$ summands give
\begin{equation}
\begin{aligned}
 &\sum_{k\geq1}\E_{\tarstt}\!\left[
  \one\{J_1=\cdots=J_k=0\}\frac1k\sum_{j=1}^k A_j\right]\\
 &\qquad=\sum_{k\geq1}\E_{\tarstt}[A\one\{J=0\}](1-\alpha)^{k-1}
 =\frac{\E_{\tarstt}[A\one\{J=0\}]}{\alpha}.
\end{aligned}
 \label{eq:singleton-renewal-sum}
\end{equation}
The strong Markov property at a return to $\tarstt$ also gives
\[
 \rho_{\tarstt}=\E_{\tarstt}A+(1-\alpha)\rho_{\tarstt},
 \qquad
 \rho_{\tarstt}=\frac{\E_{\tarstt}A}{\alpha}.
\]
The series in~\eqref{eq:singleton-renewal-sum} is therefore at most
$\rho_{\tarstt}$.  This proves~\eqref{eq:palm-orphan-bound} and
Lemma~\ref{lem:orphan}.

\paragraph{Proof of Lemma~\ref{lem:palm-reassembly}}
\label{subsec:proof-palm-reassembly}

The proof proceeds in four steps.

\medskip
\noindent\textbf{1. Which singular matrices can contribute?}
Fix a cut $c$ and consider its graph on $\subG\cup\{q\}$.
The degree balance is
\begin{equation}
 \one^\top L_c=(\kappa^c)^\top,
 \qquad
 \kappa_j^c=N_j-\sum_{i\in\subG}(\mathsf C_c)_{ij}\geq0,
 \label{eq:quotient-cut-flow}
\end{equation}
where $\one$ is the all-ones column vector and $\kappa_j^c$ is the number
of new edges from $q$ to $j$.  These edges represent the initial parts
of the rooted trajectories.

Here and below, a connected component of a directed graph means a
connected component after the edge orientations are ignored.
There is a direct linear-algebra description of the kernel of $L_c$.
For $x\in\R^{\subG}$, set $x_q=0$.  Equal indegree and outdegree imply
\begin{equation}
 x^\top L_c x
   =\frac12\sum_{e:u\to v}(x_u-x_v)^2,
 \label{eq:quotient-graph-energy}
\end{equation}
where the sum counts every edge of the new graph once, including parallel
edges and edges incident to $q$.  To verify the formula, expand the
squares: their two diagonal sums are equal because the degrees balance,
and their cross terms are exactly those of $L_c$.

If $L_cx=0$, the sum of squares vanishes.  Hence $x$ is constant on each
connected component and is zero on the component containing $q$.
Conversely, every vector of that form belongs to the kernel, directly
from the entries of $L_c$.  The same conclusion holds for $L_c^\top$.
Consequently, the nullity of $L_c$ equals the number of connected components
that do not contain $q$.

Suppose $\det L_c=0$ and $e_{\tarstt}^\top\adj(L_c)\ne0$.
A nonzero cofactor implies that $L_c$ has rank $|\subG|-1$, so there is
exactly one component $\mathcal C$ not containing $q$.
The right and left kernels are both spanned by the vector
$\one_{\mathcal C}$, which is one on $\mathcal C$ and zero elsewhere.
Since
$L_c\adj(L_c)=\adj(L_c)L_c=0$, the adjugate is a scalar multiple of
$\one_{\mathcal C}\one_{\mathcal C}^\top$.
In particular, $\tarstt\in\mathcal C$.

Set $\subG_- =\subG\setminus\{\tarstt\}$ and
$D=(L_c)_{\subG_-,\subG_-}$.
This principal submatrix is invertible: applying
\eqref{eq:quotient-graph-energy} to a vector with both $x_q=0$ and
$x_{\tarstt}=0$ leaves no nonzero constant on either component.
Moreover, $\det D\geq0$ by~\eqref{eq:boundary-root-pairing-count}, so
$\det D>0$.  If $\subG_-$ is empty, use $\det D=1$.
The $(\tarstt,\tarstt)$ cofactor identifies the scalar above, giving
\begin{equation}
 \adj(L_c)=\det D\,\one_{\mathcal C}\one_{\mathcal C}^\top,
 \qquad
 e_{\tarstt}^\top\adj(L_c)b_c
   =\det D\sum_{s\in\mathcal C}(b_c)_s.
 \label{eq:closed-class-residual}
\end{equation}
There is no edge between $\mathcal C$ and its complement; in particular,
no edge from $q$ enters $\mathcal C$.
Cuts whose target adjugate row is zero make no contribution to
\eqref{eq:orphan-term} and are discarded below.

\medskip
\noindent\textbf{2. Add the bijections at the other retained vertices.}
For a remaining cut $c$, let $\mathfrak L(c)$ be the set of additional
bijections at $\subG_-$ in its new graph such that following successors
stops at $\{\tarstt,q\}$.  By
\eqref{eq:boundary-root-pairing-count}, applied with $A=\subG_-$,
\begin{equation}
 |\mathfrak L(c)|=\det D\prod_{s\in\subG_-}(N_s-1)!.
 \label{eq:singular-lift-count}
\end{equation}
This application is valid even though $L_c$ itself is singular.
A choice $\psi\in\mathfrak L(c)$, together with $c$, specifies bijections
at all original vertices except $\tarstt$ and $q$.  The trail-replacement
correspondence in the graph lemma recovers these original bijections
uniquely.

The trails starting at $q$ stay in its component, so they give $n$ labeled
complete rooted trajectories that avoid $\tarstt$.
The edges in $\mathcal C$ split into exactly $N_{\tarstt}$ trails that
start at $\tarstt$ and return to it, with no intermediate visit to
$\tarstt$.  Expand each new edge into its ordered list of original edges.
The resulting trails are primitive returning excursions.  Their total
residual weight is $\sum_{s\in\mathcal C}(b_c)_s$.

\medskip
\noindent\textbf{3. Keep the counting factors when integrating over data.}
On the count component $\bm N$, the fresh-label law gives
\begin{equation}
 \Pp_\cT(d\tau)=\Omega_{\rm full}(\tau)\,\Xi_{\bm N}(d\tau).
 \label{eq:bag-marginal-product-measure}
\end{equation}
The product measure $\Xi_{\bm N}$ also charges data without balanced
degrees or without a full assembly.  Before using a count formula or
canceling a determinant, restrict to
$\mathcal A_\subG\cap\{\Omega_{\rm full}>0\}$.
On this set the data admit a full assembly, their degrees balance, and
$\det L>0$.  Since the correction vanishes on
$\mathcal A_\subG^c$, the triangle inequality, Tonelli's theorem, and
\eqref{eq:full-assembly-count} give
\begin{equation}
\begin{aligned}
 \E|\widehat\Delta_\subG^{\rm orp}(\tarstt)|
 &\leq\sum_{\bm N}
 \int_{\mathcal A_\subG\cap\{\Omega_{\rm full}>0\}}
 \left\{\prod_{s\in\subG}(N_s-1)!\right\}\\
 &\qquad\qquad{}\times
 \sum_{\substack{c\in\mathfrak C_R(\tau)\\\det L_c=0}}
 |e_{\tarstt}^\top\adj(L_c)b_c|\,\Xi_{\bm N}(d\tau).
\end{aligned}
 \label{eq:orphan-tonelli-bound}
\end{equation}
Indeed, multiplying~\eqref{eq:orphan-term} by
$\Omega_{\rm full}=\det L\prod_{s\in\cO}(N_s-1)!$
cancels $\det L$ and precisely the factorials at $R$.

For each choice of $c$ and $\psi$, list its $N_{\tarstt}$ returning
excursions by the increasing fresh label of their first edge at
$\tarstt$.  Let $\mathfrak S_k$ denote the set of permutations of
$\{1,\ldots,k\}$.  A permutation
$\pi\in\mathfrak S_{N_{\tarstt}}$ orders this list; write
$A_j(c,\psi,\pi)$ for the sum of the absolute primitive residuals on
excursion $j$ in that order.  Equations~\eqref{eq:closed-class-residual}
and~\eqref{eq:singular-lift-count} give
\begin{equation}
\begin{aligned}
 &\left\{\prod_{s\in\subG}(N_s-1)!\right\}
       |e_{\tarstt}^\top\adj(L_c)b_c|\\
 &\quad=
 \sum_{\psi\in\mathfrak L(c)}
 \sum_{\pi\in\mathfrak S_{N_{\tarstt}}}
 \left|\frac1{N_{\tarstt}}\sum_{s\in\mathcal C}(b_c)_s\right|\\
 &\quad\leq
 \sum_{\psi\in\mathfrak L(c)}
 \sum_{\pi\in\mathfrak S_{N_{\tarstt}}}
 \frac1{N_{\tarstt}}\sum_{j=1}^{N_{\tarstt}}A_j(c,\psi,\pi).
\end{aligned}
 \label{eq:orphan-factorial-chain}
\end{equation}
For the equality, there are
$\det D\prod_{s\in\subG_-}(N_s-1)!$ choices of $\psi$ and
$N_{\tarstt}!$ choices of $\pi$; division by $N_{\tarstt}$ gives exactly
the remaining factor $(N_{\tarstt}-1)!$.
The inequality is the triangle inequality along the expanded excursions.
Thus the factor $1/N_{\tarstt}$ is explicit; no extra labeling factor is
implicit in this calculation.

\medskip
\noindent\textbf{4. Change from edge data to paths and excursions.}
For $k=N_{\tarstt}$, trace the labeled roots and then the $k$ excursions
in order $\pi$.  This defines an injection from $(\tau,c,\psi,\pi)$
into tuples of the form
\begin{equation}
 \left(
 \begin{array}{l}
 n\text{ labeled complete rooted paths avoiding }\tarstt,\\
 k\text{ ordered returning excursions from }\tarstt,\\
 \text{source-wise assignments of fresh edge labels}
 \end{array}
 \right).
 \label{eq:palm-reassembly-bijection}
\end{equation}
Here a source-wise assignment is a bijection between the transition
occurrences starting at $s$ and the labels $1,\ldots,N_s$, for each $s$.
Every displayed transition retains its reward and successor.

To invert this map on its image, pair consecutive edges along the paths
and excursions at every vertex other than $\tarstt$ and $q$, including the first
edge of each rooted path after its labeled added edge from $q$.
Restriction to $R$ recovers $c$; after replacing its trails, restriction
to $\subG_-$ recovers $\psi$.
The labels of the first $\tarstt$-edges and the displayed excursion order
recover $\pi$.  Thus the data and all the chosen bijections are recovered.

Membership in the image is unchanged by any source-wise relabeling.
Consequently, for each fixed ordered path--excursion tuple in the image,
there are exactly $\prod_{s\in\cO}N_s!$ assignments of fresh labels.
Even when two recorded transitions agree, they are different occurrences
and receive distinct labels.  This multiplicity cancels precisely the
factor $\prod_s1/N_s!$ in $\Xi_{\bm N}$.

For fixed path lengths and state sequences, the coordinate change merely
reorders the reward coordinates and their joint kernel factors $Q_s$.
Summing over lengths and state sequences therefore leaves the product of
the same kernels along the $n$ paths and $k$ excursions, together with the
$n$ initial-state factors $\mu$.  In particular, the argument does not
factor a reward from its successor.

Finally, enlarge the image to all tuples in
\eqref{eq:palm-reassembly-bijection}, dropping the coverage and
realizability restrictions inherited from the original data.
The integrand in~\eqref{eq:orphan-factorial-chain} is nonnegative, so this
only increases the integral.  Under the resulting product path law, the
$n$ rooted trajectories are independent and each avoids $\tarstt$ with
probability $1-p_{\tarstt}$.  The $k$ excursions are independent stopped
paths from $\tarstt$, with the indicators $J_1=\cdots=J_k=0$ imposing
return.  Substituting~\eqref{eq:orphan-factorial-chain} into
\eqref{eq:orphan-tonelli-bound} and summing over $k\geq1$ gives
\eqref{eq:orphan-palm-reassembly} and Lemma~\ref{lem:palm-reassembly}.

\setcounter{secnumdepth}{3}
\subsection{Proof of Proposition~\ref{prop:lam-lower-bound}}
\label{subsec:proof-local-lower}

All quantities with a superscript zero refer to the baseline model.
If $\sigma_{{\rm TD},0}^2(\tarstt)=0$, the conclusion follows from
nonnegativity of squared loss.  Otherwise define, for
$(r,u)\in[-1,1]\times(\cS\cup\{\termState\})$,
\[
 w_s=\frac{F^0(\tarstt,s)}{\nu^0(s)},\qquad
 g_s(r,u)=w_s\{r+V^0(u)-V^0(s)\},\qquad
 B=\max_s\|g_s\|_\infty.
\]
The finite state space and positive occupancies imply $B<\infty$.
The Bellman equation gives $\int g_s\,dK_s^0=0$.  Consider the
one-parameter family
\begin{equation}
 \frac{dK_s^\theta}{dK_s^0}(r,u)
 =\exp\{\theta g_s(r,u)-A_s(\theta)\},
 \qquad A_s(\theta)=\log\int e^{\theta g_s}\,dK_s^0.
 \label{eq:lower-least-favorable-tilt}
\end{equation}
Its row log likelihood ratios have absolute value at most
$2B|\theta|$, so $K^\theta\in\mathcal N_{2B|\theta|}(K^0)$.
Rewards remain bounded by one and the absorbing state is unchanged.

We first verify uniform control of the complete trajectories.  Set
$L=\lceil2\effhorizon\rceil$.  Under the baseline model,
$\sup_s\Pp_s^0(T>L)\leq1/2$ by Markov's inequality.  A kernel in
$\mathcal N_\varepsilon(K^0)$ therefore satisfies
\[
 \sup_s\Pp_s(T>L)\leq e^{\varepsilon L}/2\leq3/4
 \quad\text{if }\varepsilon L\leq\log(3/2).
\]
The Markov property gives
$\sup_s\Pp_s(T>kL)\leq(3/4)^k$ for every integer $k\geq0$.
This proves a common exponential tail and a common effective-horizon
bound $h'\leq C\effhorizon$ in a fixed neighborhood of zero.
In particular, all values in that neighborhood lie in $[-h',h']$.

The score of one complete trajectory at zero is
\[
 Z=\sum_{t<T}g_{S_t}(R_t,S_{t+1}).
\]
Its summands are martingale differences.  Consequently its Fisher
information is
\begin{equation}
 I=\E_0Z^2
 =\sum_{s\in\cS}\nu^0(s)\int g_s^2\,dK_s^0
 =\sum_{s\in\cS}
   \frac{F^0(\tarstt,s)^2\sigma_{V^0}^2(s)}{\nu^0(s)}
 =\sigma_{{\rm TD},0}^2(\tarstt)>0.
 \label{eq:lower-tilt-information}
\end{equation}
This is a regular dominated path model.  To check quadratic-mean
differentiability directly, its log likelihood ratio satisfies
\[
 \log\frac{d\Pp_\theta}{d\Pp_0}
 =\theta Z+O(\theta^2T).
\]
Taylor's formula for its square root bounds the remainder after
$1+\theta Z/2$ by
$C\theta^2(T+T^2)e^{C|\theta|T}$.  The exponential tail just proved
makes the square of this bound integrable for sufficiently small
$|\theta|$, and its integral is $O(\theta^4)=o(\theta^2)$.
At a nearby parameter, the score increments are
$g_s-A_s'(\theta)$ and their information is
\[
 I(\theta)=\sum_{s\in\cS}\nu^\theta(s)A_s''(\theta).
\]
This function is continuous at zero: the row laws are bounded smooth
tilts and $\nu^\theta=\mu^\top(I-P^\theta)^{-1}$ is continuous while
the kernel remains transient.

Let $v(\theta)=V_{K^\theta}(\tarstt)$.  Differentiating the finite
Bellman system at zero gives
\begin{align}
 v'(0)
 &=\sum_{s\in\cS}F^0(\tarstt,s)
       \int\{r+V^0(u)\}g_s(r,u)\,dK_s^0(r,u)\notag\\
 &=\sum_{s\in\cS}F^0(\tarstt,s)w_s\sigma_{V^0}^2(s)
 =I.
 \label{eq:lower-tilt-value-derivative}
\end{align}
The derivative is continuous on a sufficiently small parameter interval.

For completeness, a Bayesian information inequality now yields the
local lower bound without a separate asymptotic approximation theorem.
Take the density
$\pi(u)=\frac{15}{16}(1-u^2)^2\one\{|u|<1\}$, which vanishes at
both endpoints and satisfies
$\int(\pi')^2/\pi=10$.
For fixed $b>0$, put
$\pi_{n,b}(\theta)=\frac{\sqrt n}{b}\pi(\sqrt n\theta/b)$.
For sufficiently large $n$ its support lies in the regular parameter
interval.  Project any estimator onto $[-h',h']$; this cannot increase
its risk anywhere in this family.  If $S_{n,\theta}$ denotes the
likelihood score for $n$ trajectories, integration by parts, with
expectation over the parameter prior and the sample, gives
\[
 \E\left[(\widehat v_n-v(\theta))
   \left\{S_{n,\theta}+\frac{\pi_{n,b}'(\theta)}
                              {\pi_{n,b}(\theta)}\right\}\right]
 =\int v'(\theta)\pi_{n,b}(\theta)\,d\theta.
\]
There is no boundary term because the prior density vanishes at its
endpoints.  Boundedness of the projected estimator and the uniform
exponential tail justify differentiation and integration.  Conditional
mean zero of the likelihood score eliminates the cross term, so
Cauchy--Schwarz yields
\begin{equation}
 \sup_{|\theta|\leq b/\sqrt n}
 n\E_\theta[\widehat v_n-v(\theta)]^2
 \geq
 \frac{\left\{\int v'(\theta)\pi_{n,b}(\theta)\,d\theta\right\}^2}
 {\int I(\theta)\pi_{n,b}(\theta)\,d\theta+10/b^2}.
 \label{eq:lower-bayesian-information}
\end{equation}
The same lower bound holds for the original unprojected estimator.
Continuity and~\eqref{eq:lower-tilt-value-derivative} show that the
liminf of the right side is $I^2/(I+10/b^2)$.
The parameter interval is contained in
$\mathcal N_{2Bb/\sqrt n}(K^0)$ by
\eqref{eq:lower-least-favorable-tilt}.  Taking the supremum over $b>0$
and using~\eqref{eq:lower-tilt-information} proves
\eqref{eq:local-asymptotic-lower}, completing the proof.

\subsection{Proof of Theorem~\ref{thm:finite-sample-lower-bound}}
\label{subsec:proof-finite-lower}

Write $\mathcal R_n$ for the minimax risk on the left side of
\eqref{eq:finite-sample-lower}.  We prove its two lower bounds using
different subfamilies.  In each subfamily the initial distribution is
fixed, so the argument also applies when that distribution is known.
It suffices to take $c_1\geq1$ and $c_2\geq256$; increasing these
universal constants does not affect the conclusion.

\paragraph{The asymptotic-variance term.}
Let $a=1/\sigma_*^2$ and $\beta=\sigma_*/(16\sqrt n)$.
Start a trajectory at $\tarstt$ with probability $a$ and at one other
state with probability $1-a$, and terminate after one transition.
The reward at the other state is zero.  At $\tarstt$ it is Bernoulli
with mean $1/2+z\beta$, where $z\in\{-1,1\}$ indexes two models.
Give every remaining state zero reward and an immediate transition to
$\termState$.
Both models belong to the class: their target occupancy is
$a>\nu_0$, their TD variance is at most $1/(4a)=\sigma_*^2/4$,
and their probability of entering a rare nonterminal state from the
target is zero.  The sample-size condition gives $\beta<1/4$.

For $|\beta|\leq1/4$, the Bernoulli divergence is at most
$32\beta^2$.  Thus the Kullback--Leibler divergence between the two
sample laws is at most $32na\beta^2=1/8$, and their total variation
distance is at most $1/4$ by Pinsker's inequality.  Classifying an
estimate according to whether it lies above or below $1/2$ gives
\begin{equation}
 \mathcal R_n\geq\frac{\beta^2}{2}(1-1/4)
 =\frac{3\sigma_*^2}{2048n}.
 \label{eq:lower-standard-component}
\end{equation}
Here the factor $1/2$ is the equal-prior testing factor: on a
classification error the squared estimation error is at least
$\beta^2$.

\paragraph{The unresolved-excursion term.}
If $q/\nu_0\leq16\sigma_*^2$, the desired sum is at most
$17\sigma_*^2$, so~\eqref{eq:lower-standard-component} already proves
the theorem.  We may therefore assume
\begin{equation}
 q/\nu_0>16\sigma_*^2>80.
 \label{eq:lower-nontrivial-regime}
\end{equation}
Set
\[
 m=\lfloor1/\nu_0\rfloor,\qquad
 N=2\left\lfloor\frac{d-m}{2}\right\rfloor,\qquad
 \alpha=q/2,\qquad
 \varepsilon=\frac1{32}\sqrt{\frac{m}{n\alpha}}.
\]
The assumptions imply $m\geq80$,
$\frac1{2\nu_0}\leq m\leq\sqrt n$, and
$N\geq d/2\geq64n^2$.
Moreover $n\alpha/m\geq q/(2\nu_0)>40$, so
$0<\varepsilon<1/4$.
Use $m$ initial states $s_0=\tarstt,s_1,\ldots,s_{m-1}$ and $N$
leaf states $u_1,\ldots,u_N$.  If one state remains, give it an
immediate terminal transition and zero initial mass and reward.

Let $\psi\in\{-1,1\}^N$ be balanced, with exactly $N/2$ entries of
each sign, and let $\zeta\in\{-1,1\}^m$.  The initial distribution
is uniform on the $m$ initial states.  At $s_i$ the reward is zero,
and the transition probabilities are
\begin{equation}
 P(s_i,\termState)=1-\alpha,\qquad
 P(s_i,u_j)=\frac{2\alpha}{N}\one\{\psi_j=\zeta_i\}.
 \label{eq:lower-hard-transitions}
\end{equation}
Every leaf terminates in one step, receiving a reward in $\{-1,1\}$
with mean $\varepsilon\psi_j$.  Thus $T\leq2$, rewards are bounded
by one, and
\[
 V(u_j)=\varepsilon\psi_j,\qquad
 V(s_i)=\alpha\varepsilon\zeta_i,\qquad
 \nu(s_i)=1/m.
\]
Writing $k=|\{i:\zeta_i=\zeta_0\}|$, the TD variance at the target
is exactly
\begin{equation}
 \sigma_{\rm TD}^2(\tarstt)
 =m\alpha(1-\alpha)\varepsilon^2
   +\frac{\alpha m}{k}(1-\varepsilon^2).
 \label{eq:lower-hard-td-variance}
\end{equation}
For $k\geq m/4$ it is at most
$m^2/(1024n)+4\alpha<5<\sigma_*^2$.
Every leaf has occupancy at most
$2\alpha/N=q/N\leq2/d=\rho$, and the target enters the leaf layer
with probability $\alpha\leq q$.  Hence all instances with
$k\geq m/4$ belong to~\eqref{eq:finite-lower-class}.

For a sign $z$, take independent uniform priors on the balanced vectors
$\psi$ and on vectors $\zeta$ with $\zeta_0=z$.
Let $Q_z$ be the resulting mixture of the $n$-trajectory laws, and
$\widetilde Q_z$ the mixture conditioned on $k\geq m/4$.
In the latter mixture every target value equals $z\alpha\varepsilon$.
Since $k=1+\operatorname{Binom}(m-1,1/2)$ under either prior,
Hoeffding's inequality gives
\begin{equation}
 d_{\rm TV}(Q_z,\widetilde Q_z)
 \leq\Pp\{k<m/4\}
 \leq\exp\left\{-\frac{(m+2)^2}{8(m-1)}\right\}
 \leq e^{-m/8}.
 \label{eq:lower-prior-truncation}
\end{equation}
The first inequality follows by coupling a prior and its conditioning
and then generating the same sample whenever their parameters agree.

The remaining comparison removes the information in the rare-state
labels.  Define an auxiliary experiment as follows.  Draw the same
prior on $\zeta$; each trajectory independently starts at a uniform
initial state and exits with probability $\alpha$.  On exit, choose a
uniform label from $\{1,\ldots,N\}$ independently of all previous
labels, and draw a $\{-1,1\}$ reward of mean
$\varepsilon\zeta_i$ if the starting state is $s_i$.
Let $\widehat Q_z$ be its mixture law.  This auxiliary experiment need
not be an MRP because its reward can depend on the preceding initial
state; it is used only to compare statistical experiments.

The following lemma quantifies the loss of information when collisions
among the rare-state labels are unlikely.
\begin{lemma}[Comparison of the two mixture experiments]
\label{lem:lower-mixture-comparison}
For the construction above, with $N\geq2n$ and
$0<\varepsilon\leq1/2$,
\begin{align}
 d_{\rm TV}(Q_z,\widehat Q_z)&\leq\frac{2n^2}{N},
 \qquad z\in\{-1,1\},\label{eq:lower-label-comparison}\\
 d_{\rm TV}(\widehat Q_1,\widehat Q_{-1})
 &\leq\sqrt{\frac{2n\alpha\varepsilon^2}{m}}.
 \label{eq:lower-ideal-comparison}
\end{align}
\end{lemma}
\par\noindent
See Subsubsection~\ref{subsubsec:proof-lower-mixture-comparison} for the proof.

Using the lemma and~\eqref{eq:lower-prior-truncation}, we obtain
\[
 d_{\rm TV}(\widetilde Q_1,\widetilde Q_{-1})
 \leq2e^{-m/8}+\frac{4n^2}{N}+\frac{\sqrt2}{32}
 <\frac14.
\]
All hypotheses used here hold because $m\geq80$, $N\geq64n^2$,
and the chosen $\varepsilon$ is below $1/4$.
The testing argument with the two values
$\pm\alpha\varepsilon$ therefore gives
\begin{equation}
 \mathcal R_n\geq\frac{3\alpha^2\varepsilon^2}{8}
 =\frac{3\alpha m}{8192n}
 \geq\frac{3q}{32768n\nu_0}.
 \label{eq:lower-excursion-component}
\end{equation}
For the same family and $\subG=\{s_0,\ldots,s_{m-1}\}$, the
collapsed transition matrix is zero.  The collapsed return has second
moment $\alpha$ and mean $\alpha\varepsilon\zeta_i$. Since each initial
state has occupancy $1/m$, the subgraph Bellman variance is
\begin{equation}
 \sigma_\subG^2(\tarstt)
 =m\bigl(\alpha-\alpha^2\varepsilon^2\bigr)
 \asymp \frac q{\nu_0},
 \label{eq:lower-family-collapsed-variance}
\end{equation}
where the comparison holds in the nontrivial regime
\eqref{eq:lower-nontrivial-regime} used in this construction. The rare
states have occupancy at most $2/d\ll1/n$, whereas the full-TD
asymptotic variance is bounded by five. The cost of these unresolved
excursions therefore appears both in the subgraph benchmark and in the
information-theoretic lower bound.
Combining~\eqref{eq:lower-standard-component} and
\eqref{eq:lower-excursion-component}, and the already treated case
$q/\nu_0\leq16\sigma_*^2$, proves
\eqref{eq:finite-sample-lower} with a universal $c>0$.
This completes the proof of Theorem~\ref{thm:finite-sample-lower-bound}.

\subsubsection{Proof of Lemma~\ref{lem:lower-mixture-comparison}}
\label{subsubsec:proof-lower-mixture-comparison}

Fix the initial-state sign vector $\zeta$.  After integrating over the
balanced leaf coloring $\psi$, expose leaf labels in order of their
occurrence.  Suppose that $r$ different labels have appeared and
$r_a$ of them have color $a\in\{-1,1\}$.  If the next exiting
trajectory starts from a state of sign $a$, its conditional probability
of selecting an old label is exactly $2r_a/N\leq2r/N$.
Conditionally on selecting a new label, that label is uniform over the
$N-r$ unused labels.  Indeed, the unexposed coloring is uniform over
all assignments with the prescribed remaining color counts; observed
rewards depend only on the already exposed colors and do not change
this uniformity.

Compare this experiment with one that always uses a new, uniformly
chosen label on exit, while keeping the same starts, exit indicators,
and reward draws.  Couple them sequentially until the first repeated
label in the original experiment.  Before that time their fresh labels
have the same uniform law, and their rewards have the same mean
$\varepsilon\zeta_i$.  Summing the conditional collision
probabilities over at most $n$ exits bounds the failure probability by
\[
 \sum_{r=0}^{n-1}\frac{2r}{N}=\frac{n(n-1)}{N}.
\]
The condition $N\geq2n$ guarantees that a fresh label of either color
is available throughout a successful coupling.
In turn, sampling fresh uniform labels can be coupled with independent
uniform labels with failure probability at most
$n(n-1)/(2N)$.  Applying the triangle inequality and then averaging
over the common prior on $\zeta$ gives
\[
 d_{\rm TV}(Q_z,\widehat Q_z)
 \leq\frac{3n(n-1)}{2N}\leq\frac{2n^2}{N},
\]
which proves~\eqref{eq:lower-label-comparison}.

For~\eqref{eq:lower-ideal-comparison}, augment both auxiliary
experiments with the latent signs $\zeta_1,\ldots,\zeta_{m-1}$.
These signs, the initial states, the exit indicators, and the leaf
labels have the same law under the two experiments.  Conditionally
on them, the reward laws differ only on trajectories that start at
$s_0$ and exit.  The expected number of these trajectories is
$n\alpha/m$.  The chain rule for relative entropy, followed by its
monotonicity under removal of the latent signs, gives
\begin{align*}
 D_{\rm KL}(\widehat Q_1\|\widehat Q_{-1})
 &\leq\frac{n\alpha}{m}
 D_{\rm KL}\left(
  \operatorname{Ber}\left(\frac{1+\varepsilon}{2}\right)
  \middle\|
  \operatorname{Ber}\left(\frac{1-\varepsilon}{2}\right)
 \right)\\
 &=\frac{n\alpha}{m}\varepsilon
       \log\frac{1+\varepsilon}{1-\varepsilon}
 \leq\frac{4n\alpha\varepsilon^2}{m}.
\end{align*}
The last inequality uses $\varepsilon\leq1/2$.
Pinsker's inequality now proves~\eqref{eq:lower-ideal-comparison},
completing the proof of the lemma.

\section{Discussion and conclusion}\label{sec:discussion}

We asked whether TD retains a useful guarantee when most of the state
space is unobserved, a regime in which its empirical linear system looks
poorly determined. The answer is that TD does not need global coverage at
all. At every sample size, its error at a target is governed by the
smallest subgraph Bellman variance that the data support: the Monte Carlo
variance when only the target is well visited, and benchmarks approaching
the efficient TD variance as the trajectories cover more of the target's
future. The estimator is unchanged throughout; the interpolation between
rollout and local pooling happens inside the empirical Bellman equation.
Behind this adaptivity is a conditional-averaging identity: up to an
exponentially small correction, full TD is the conditional expectation of
the subgraph Bellman estimator given the sufficient statistics of the
sample, so it keeps the reliability of every well-covered coarse estimator
while averaging out their excess noise.

The lower bounds mark the limit of this adaptivity. The efficient variance
cannot be attained uniformly before the rarely visited part of the state
space is resolved, and the unavoidable price, an exit-probability term
divided by the target occupancy, is the scale that the subgraph Bellman
variance assigns to the hard instances. The obstruction is informational
rather than specific to TD. The lower bound is a worst-case statement over
a class of models and does not identify the exact minimax value of the
oracle functional for an individual MRP.

Several directions follow from this viewpoint. With function approximation,
the analogue of a covered subgraph is a region of the state space or a set
of features on which the data determine the Bellman equation, and the
question is whether TD adapts to such local structure as well. In control
and offline reinforcement learning, where coverage conditions are usually
stated for the whole state space or for the distribution of a policy, our
results suggest a target-specific notion of coverage that may extend to
$Q$-functions, policy improvement, and the choice of which trajectories to
collect for evaluating a given decision. The analysis here uses independent
complete trajectories and the exact empirical Bellman solution; single long
trajectories, truncated or off-policy data, online stochastic
approximation, and multi-step methods such as TD$(\lambda)$ each change
what collapsing at a subgraph means, and the conditional-averaging argument
may transfer to some of them. Finally, the conditional-expectation techniques
for Markov chains developed in our proofs are of independent interest, and
may be useful beyond reinforcement learning, for example in the analysis of
queueing networks.

\section*{Acknowledgement}
This work was partially supported by NSERC grant RGPIN-2024-05092 to WM. JQ acknowledges support from ARO through award W911NF-21-1-0328 and the Simons Foundation and the NSF through award DMS-2031883. The authors would like to thank Martin J. Wainwright, Fangzhou Su, Siva Theja Maguluri, and Zaiwei Chen for helpful discussions.

\bibliography{refs}

\appendix
\section{Omitted proofs in the proof of Proposition~\ref{prop:collapsed}}
\label{app:omitted-collapsed}

This appendix contains the omitted lemma proofs used in
Subsection~\ref{subsec:proof-collapsed}.

\subsection{Proof of Lemma~\ref{lem:tail}}
\label{app:proof-empirical-resistance}\label{app:horizon-tail}

We use the following consequences of assumption~\eqref{ass:effective-horizon}.

\begin{lemma}[Effective-horizon consequences]\label{lem:horizon-consequences}
Under assumption~\eqref{ass:effective-horizon}, the hitting probabilities
$p_s$ defined in Section~\ref{sec:model-and-result} satisfy, for every state $s$,
\begin{equation}
 \nu(s)=p_s F(s,s),\qquad F(s,s)\leq \effhorizon,\qquad
 p_s\geq\frac{\nu(s)}\effhorizon.
 \label{eq:occupancy-hit-probability}
\end{equation}
For every nonempty deterministic subgraph $\subG$ with $\nuMin>0$,
\begin{equation}
 |\subG|\nuMin\leq\sum_{s\in \subG}\nu(s)\leq\E_\mu T\leq \effhorizon,
 \qquad |\subG|\leq \effhorizon/\nuMin.
 \label{eq:covered-subgraph-cardinality}
\end{equation}
The occupancy of the collapsed chain at $s\in\subG$ is also $\nu(s)$.
For every state $s$ and $t\geq e\effhorizon$,
\begin{equation}
 \Pp_s(T>t)\leq e\exp\{-t/(e\effhorizon )\}.
 \label{eq:effective-horizon-tail}
\end{equation}
\end{lemma}
\par\noindent
See Subsubsection~\ref{app:proof-horizon-consequences} for the proof.

Let $k=N-\mathsf C\one\geq0$ be terminal counts and
$r=N-\mathsf C^\top\one\geq0$ be root counts.  The symmetric part of $L$ has
Dirichlet form
\[
 z^\top\frac{L+L^\top}{2}z
 =\frac12\sum_{s,u}\mathsf C_{su}(z_s-z_u)^2
  +\frac12\sum_s(k_s+r_s)z_s^2.
\]
Every observed state has an undirected path of at most $d$ positive-count
edges to ground.  Its effective resistance is at most $2(d+1)$.  If
$z=L^{-1}e_s$, then
\[
 z_s=z^\top Lz=z^\top\frac{L+L^\top}{2}z.
\]
Cauchy--Schwarz in this energy norm gives
$(L^{-1})_{ss}\leq2(d+1)$.  With
$\widehat F=(I-D_N^{-1}\mathsf C)^{-1}$, we have
$L^{-1}=\widehat FD_N^{-1}$, and the Green hitting identity
$\widehat F(\tarstt,s)=\Pp_{\tarstt}(\widehat\tau_s<\widehat\tau_{\termState})
\widehat F(s,s)$ proves the off-diagonal inequality for observed $\tarstt$; if
$\tarstt\notin\cO$, then $\widehat V_{\rm TD}(\tarstt)=0$ and its error is at most $\effhorizon$.

Let $b$ denote the sum of the pathwise residual vectors, namely the
unnormalized primitive Bellman residual vector or $\sum_i\varepsilon_i$
for the collapsed sample.  In either case,
$\sum_s|b_s|\leq3\effhorizon M_{\rm tot}$.  In the full graph $d\leq M_{\rm tot}$; in
the collapsed graph $d\leq|\subG|\leq \effhorizon/\nuMin$.  Apply
inequality~\eqref{eq:path-laplacian-resistance}, and use
$\norm{M_{\rm tot}}_{L^p}\leq pn\effhorizon $, to obtain
\eqref{eq:full-td-fourth-moment}--\eqref{eq:collapsed-fourth-moment}.
This proves Lemma~\ref{lem:tail}.

\subsubsection{Proof of Lemma~\ref{lem:horizon-consequences}}
\label{app:proof-horizon-consequences}

The strong Markov property at the first visit to $s$ gives
identity~\eqref{eq:occupancy-hit-probability}, with $F(s,s)\leq\E_sT\leq\effhorizon$.
Summing the expected visits over $s\in\subG$ gives
inequality~\eqref{eq:covered-subgraph-cardinality}.
Since a collapsed visit to $s$ is an ordinary visit to $s$, the occupancy of
the collapsed chain is still $\nu(s)$.
For $t\geq e\effhorizon$, Markov's inequality with
$p=\lfloor t/(e\effhorizon)\rfloor$ gives
inequality~\eqref{eq:effective-horizon-tail}. This proves
Lemma~\ref{lem:horizon-consequences}.

\subsection{Proof of Lemma~\ref{lem:good}}
\label{subsec:proof-good}

Since $n\leq e^{\ell_\subG}$ and $|\subG|\leq \effhorizon/\nuMin\leq e^{\ell_\subG}$,
inequality~\eqref{eq:effective-horizon-tail} and a union bound imply that
Part~\ref{item:good-length} fails with
probability at most $e^{-26\ell_\subG}$ if $C_0$ is sufficiently large.

Let $N_i(s)$ be the number of visits to $s$ in trajectory $i$.  Conditional
on hitting $s$, this count is geometric on $\{1,2,\ldots\}$ with mean
$F(s,s)\leq \effhorizon$.  For every integer $p\geq2$, its rising factorial
moment gives
\[
 \E[N_i(s)^p\mid N_i(s)>0]
 \leq\E\!\left[\prod_{j=0}^{p-1}(N_i(s)+j)\,\middle|\,N_i(s)>0\right]
 =p!F(s,s)^p.
\]
The equality follows by summing the geometric series.  Multiplying by
$p_s$ and using $p_sF(s,s)=\nu(s)$ yields
\[
 \E N_i(s)^p\leq p!\effhorizon^{p-1}\nu(s),\qquad p\geq2.
\]
Scalar Bernstein, with the one-coordinate failure probability set to
$e^{-27\ell_\subG}$, therefore gives, simultaneously over $s\in \subG$,
\[
 |N_s-n\nu(s)|
 \leq C\{\sqrt{n\effhorizon \nu(s)\ell_\subG}+\effhorizon\ell_\subG\}.
\]
The union bound costs at most $|\subG|\leq e^{\ell_\subG}$.  On
Part~\ref{item:good-length}, deleting at most two paths changes a count by at
most $2L_0$.  Condition~\eqref{eq:oracle-sample-size} now gives
Part~\ref{item:good-counts}.

Let $f=F_\subG\one$, with $f(\termState)=0$.  It is deterministic,
$1\leq f\leq \effhorizon$ on $\subG$, and $P_\subG f=f-\one$.  Concatenate the collapsed paths
in lexicographic order of trajectory and within-trajectory time.  For fixed
$s$, enumerate its successive visits in this concatenation.  Immediately
before the successor at the $k$th visit is revealed, let $\mathcal F_{s,k}$
contain all previously revealed transitions and the current state $s$, and
put
\[
 Z_{s,k}=f(X_{s,k}^+)-P_\subG f(s),\qquad 1\leq k\leq N_s,
\]
where $X_{s,k}^+$ is that successor.  After the concatenated sample ends, pad
the sequence by zeros.  The strong Markov property makes the resulting
successor-revealing sequence a stopped martingale-difference sequence.  Its
increments are bounded by $2\effhorizon$, and its predictable quadratic variation at
the stopping time satisfies
\[
 \mathcal V_s:=\sum_{k\leq N_s}
 \E[Z_{s,k}^2\mid\mathcal F_{s,k}]
 \leq4\effhorizon^2N_s.
\]
On the count event already constructed, $N_s\leq2n\nu(s)$ and hence
$\mathcal V_s\leq8\effhorizon^2n\nu(s)$.  Freedman's inequality with this deterministic
variance threshold gives
\[
 \Pp\!\left(
  \left|\sum_{k\leq N_s}Z_{s,k}\right|
  >C\effhorizon \{\sqrt{n\nu(s)\ell_\subG}+\ell_\subG\},
  N_s\leq2n\nu(s)
 \right)\leq2e^{-28\ell_\subG}.
\]
A union bound over $s\in \subG$ costs at most $|\subG|\leq e^{\ell_\subG}$.  On the length event,
deleting at most two
trajectories removes at most $2L_0$ increments, each bounded by $2\effhorizon$.
Combining this deterministic deletion bound with the count bound gives,
uniformly over $s$ and $|I|\leq2$,
\[
 |(\widehat P_\subG^{(-I)}-P_\subG)f(s)|
 \leq C\effhorizon \left\{
       \sqrt{\frac{\ell_\subG}{n\nu(s)}}
       +\frac{\ell_\subG+L_0}{n\nu(s)}
       \right\}\leq\frac12.
\]
Hence
\[
 \widehat P_\subG^{(-I)}f\leq f-\tfrac12\one.
\]
A finite substochastic matrix satisfying this drift inequality has no closed
class and is transient.  Multiplication by its nonnegative inverse gives
$\widehat F_\subG^{(-I)}\one\leq2f\leq2\effhorizon\one$, proving
Part~\ref{item:good-horizon}.

It remains to prove inequality~\eqref{eq:delete-two-inverse-bound}.  Put
$d_s=N_s^{(-I)}$.  Complete-path flow balance
gives $d^\top\widehat P_\subG^{(-I)}\leq d^\top$.  Jensen's inequality therefore makes
$\widehat P_\subG^{(-I)}$ a contraction on $L^2(d)$.  Moreover, if
$k=\lceil16\effhorizon\rceil$, monotonicity of the survival probabilities gives
\[
 (k+1)(\widehat P_\subG^{(-I)})^k\one
 \leq\sum_{j=0}^k(\widehat P_\subG^{(-I)})^j\one
 \leq\widehat F_\subG^{(-I)}\one\leq2\effhorizon\one,
\]
so $\norm{(\widehat P_\subG^{(-I)})^k}_{\infty\to\infty}\leq1/4$.
Its $L^1(d)$ norm is at most one; interpolation yields
\[
 \norm{(\widehat P_\subG^{(-I)})^k}_{L^2(d)\to L^2(d)}\leq1/2.
\]
Consequently
\[
 \norm{\widehat F_\subG^{(-I)}}_{L^2(d)\to L^2(d)}
 \leq\frac{k}{1-1/2}\leq C\effhorizon .
\]
Finally,
\[
 D_\nu^{1/2}(\widehat A^{(-I)})^{-1}D_\nu^{1/2}
 =nD_\nu^{1/2}D_d^{-1/2}
   \bigl(D_d^{1/2}\widehat F_\subG^{(-I)}D_d^{-1/2}\bigr)
   D_d^{-1/2}D_\nu^{1/2},
\]
and Part~\ref{item:good-counts} bounds both diagonal factors.  This proves
Part~\ref{item:good-inverse}.  The length, count, and drift failure
probabilities sum to at most $3e^{-26\ell_\subG}\leq e^{-24\ell_\subG}$ because $e^{\ell_\subG}\geq e$,
proving inequality~\eqref{eq:good-event-probability} and
Lemma~\ref{lem:good}.

\subsection{Proof of Lemma~\ref{lem:target-weighted-leverage}}
\label{app:proof-target-weighted-leverage}

We retain the notation and resolvent identities from
Paragraph~\ref{subsubsec:proof-leverage}.

The two resolvent identities,
inequality~\eqref{eq:target-weighted-path-controls}, and positivity of the
empirical inverses give the bound in Lemma~\ref{lem:target-weighted-leverage}.
Indeed, the contribution with $\widehat a^{(-i)}$ replaced by $\widehat a$
is bounded by that expression without $\kappa_i$; substituting
$\widehat a^{(-i)}-\widehat a=(\widehat A-L_i/n)^{-\top}(L_i/n)^\top\widehat a$
adds a term bounded by the contribution with $\kappa_i$.  This completes the
proof of Lemma~\ref{lem:target-weighted-leverage}.

\subsection{Proof of Lemma~\ref{lem:leave}}
\label{subsec:proof-leave-one}

We retain the notation from Subsubsection~\ref{app:collapsed-technical}.

Both perturbation proofs use the following residual-variance bound and
path-energy lemma.

For one collapsed transition started at $s$, put
\[
 \xi=\sum_{t<T_\subG}\delta_t,\qquad \E_s\xi^2=\sigma_{\subG,\Vstar}^2(s).
\]
If
$U_s=\sum_{t<T_\subG}R_t+\Vstar(S_{T_\subG})$, then the stopped Bellman
identity gives $\E_sU_s=\Vstar(s)$ and $\xi=U_s-\Vstar(s)$.  Thus
$\sigma_{\subG,\Vstar}^2(s)=\Var_s(U_s)$.  Since $T_\subG\leq T$ and
$\norm{T}_{L^2(\Pp_s)}\leq2\effhorizon$, Minkowski's inequality gives
\begin{equation}
 \sigma_{\subG,\Vstar}^2(s)\leq\E_sU_s^2
 \leq\left(\norm{T_\subG}_{L^2(\Pp_s)}+\effhorizon\right)^2
 \leq9\effhorizon^2.
 \label{eq:collapsed-residual-variance-bound}
\end{equation}

The inverse perturbation argument uses the following unconditional
trajectory-energy estimate.

\begin{lemma}[Unconditional path energy]\label{lem:energy}
For every deterministic $v\in\R^\subG$,
\begin{align}
 \E\norm{D_\nu^{-1/2}L_i^\top v}_2^2
 &\leq\frac{C\effhorizon }{\nuMin}\norm v_{D_\nu}^2,
 \label{eq:path-energy-laplacian}\\
 \E\norm{D_\nu^{-1/2}(L_i-A)^\top v}_2^2
 &\leq\frac{C\effhorizon }{\nuMin}\norm v_{D_\nu}^2.
 \label{eq:path-energy-centered}
\end{align}
\end{lemma}
\par\noindent
See Subsubsection~\ref{app:proof-energy} for the proof.

Since $\nuMin\leq\effhorizon$ and $\ell_\subG\geq1$, the sample-size
condition gives $n\geq C\effhorizon^2\ell_\subG^4\geq4$ after enlarging
the universal constant.  Thus $(n-|I|)/n\geq1/2$ for every $|I|\leq2$.
For $I\subseteq[n]$ with $|I|\leq2$, put
$\rho_I=(n-|I|)/n$ and $a_I^0=\rho_I^{-1}a$.
On $\cE$ the exact identity
\begin{equation}
 (\widehat A^{(-I)})^\top(\widehat a^{(-I)}-a_I^0)
 =-(\widehat A^{(-I)}-\rho_IA)^\top a_I^0.
 \label{eq:training-influence-identity}
\end{equation}
holds.  Apply inequality~\eqref{eq:delete-two-inverse-bound}, then
inequality~\eqref{eq:path-energy-centered}, to the independent centered
average on the right.  In particular,
$\Lambda_\subG^\star\preceq9\effhorizon^2D_\nu$, the squared normalized inverse contributes $C\effhorizon^2$,
and the centered path energy contributes $C\effhorizon /(n\nuMin)$.  Therefore
\[
 \E\bigl[\one_{\cE}
  \norm{\widehat a^{(-i)}-a_{\{i\}}^0}_{\Lambda_\subG^\star}^2\bigr]
 \leq C\effhorizon^2\cdot \effhorizon^2\cdot\frac{\effhorizon}{n\nuMin}
          \norm{a_{\{i\}}^0}_{D_\nu}^2
 \leq\frac{C\effhorizon^7}{n\nuMin^2},
\]
where the last inequality uses
inequality~\eqref{eq:population-influence-norm}.  The deterministic scaling
difference satisfies
\[
 \norm{a_{\{i\}}^0-a}_{\Lambda_\subG^\star}^2
 \leq\frac{C\effhorizon^2}{n^2}\norm a_{D_\nu}^2
 \leq\frac{C\effhorizon^4}{n^2\nuMin}.
\]
Consequently
\[
 \E\bigl[\one_{\cE}\norm{\widehat a^{(-i)}-a}_{\Lambda_\subG^\star}^2\bigr]
 \leq \frac{C\effhorizon^7}{n\nuMin^2}+\frac{C\effhorizon^4}{n^2\nuMin}.
\]
The common projection bounds the complement by
$C\effhorizon^4\Pp(\cE^c)/\nuMin$, proving
inequality~\eqref{eq:stabilized-leave-one-risk} and completing the proof of
Lemma~\ref{lem:leave}.

\subsubsection{Proof of Lemma~\ref{lem:energy}}
\label{app:proof-energy}

For one collapsed path,
\begin{equation}
 (L_i^\top v)_s
 =\one\{X_{i,0}=s\}v_s
 +\sum_{t\geq1:X_{i,t}=s}(v_s-v_{X_{i,t-1}}).
 \label{eq:path-laplacian-coordinate}
\end{equation}
Let $N_i(s)$ be its visit count and $M_{us}$ its number of $u\to s$
transitions.  Ordered pairs of visits give
\begin{align}
 \E N_i(s)^2&=\nu(s)(2F_\subG(s,s)-1)\leq2\effhorizon\nu(s),
 \label{eq:collapsed-visit-second-moment}\\
 \E[N_i(s)M_{us}]
 &=\nu(s)F_\subG(s,u)P_\subG(u,s)+F_\subG(s,s)\nu(u)P_\subG(u,s).
 \label{eq:visit-arrival-mixed-moment}
\end{align}
The two terms in identity~\eqref{eq:visit-arrival-mixed-moment} correspond
respectively to an $s$-visit before the marked arrival $u\to s$, and to that
arrival together with later $s$-visits.  Also
\[
 \sum_sF_\subG(s,u)P_\subG(u,s)=(P_\subG F_\subG)_{uu}=F_\subG(u,u)-1,
 \qquad
 \sum_s\frac{P_\subG(u,s)}{\nu(s)}\leq\frac1{\nuMin}.
\]
Equivalently,
\[
 (L_i^\top v)_s=N_i(s)v_s-\sum_uM_{us}v_u.
\]
Since $\sum_uM_{us}\leq N_i(s)$, weighted Cauchy--Schwarz gives
\[
 \E(L_i^\top v)_s^2
 \leq2\E N_i(s)^2v_s^2
    +2\sum_u\E[N_i(s)M_{us}]v_u^2.
\]
Sum this inequality after division by $\nu(s)$.  Identities
\eqref{eq:collapsed-visit-second-moment}--
\eqref{eq:visit-arrival-mixed-moment}, together with
\[
 \sum_sF_\subG(s,u)P_\subG(u,s)\leq \effhorizon,
 \qquad
 \sum_s\frac{P_\subG(u,s)}{\nu(s)}\leq\frac1{\nuMin},
\]
give
\[
 \sum_s\frac{\E(L_i^\top v)_s^2}{\nu(s)}
 \leq\frac{C\effhorizon }{\nuMin}\sum_u\nu(u)v_u^2.
\]
This proves inequality~\eqref{eq:path-energy-laplacian}.
Inequality~\eqref{eq:path-energy-centered} follows from the Hilbert-space
variance identity because $A=\E L_i$, proving Lemma~\ref{lem:energy}.

\subsection{Proof of Lemma~\ref{lem:leave-two}}
\label{subsec:proof-leave-two}

We retain the notation from Subsubsection~\ref{app:collapsed-technical}
and use Lemma~\ref{lem:energy} and the residual-variance bound
\eqref{eq:collapsed-residual-variance-bound}, stated in
Subsection~\ref{subsec:proof-leave-one}.

For the add-one difference, on $\cE$,
\begin{equation}
 \widehat a^{(-i)}-\widehat a^{(-i,-j)}
 =-(\widehat A^{(-i)})^{-\top}(L_j/n)^\top\widehat a^{(-i,-j)}.
 \label{eq:add-one-inverse-identity}
\end{equation}
Inside the event indicator, replace $\widehat a^{(-i,-j)}$ by
$\check a^{(-i,-j)}$; they are equal there.  Then drop the indicator.  The vector
$\check a^{(-i,-j)}$ is independent of trajectory $j$ and is uniformly bounded in
$D_\nu$-norm.  Inequalities~\eqref{eq:delete-two-inverse-bound},
\eqref{eq:path-energy-laplacian}, and
\eqref{eq:collapsed-residual-variance-bound} give
\[
 \E\bigl[\one_{\cE}
  \norm{\widehat a^{(-i)}-\widehat a^{(-i,-j)}}_{\Lambda_\subG^\star}^2\bigr]
 \leq C\effhorizon^2\cdot \effhorizon^2\cdot\frac1{n^2}
          \frac{\effhorizon}{\nuMin}\E\norm{\check a^{(-i,-j)}}_{D_\nu}^2
 \leq\frac{C\effhorizon^7}{n^2\nuMin^2}.
\]
The complement is
again controlled by the common projection.  This proves
inequality~\eqref{eq:stabilized-leave-two-difference}, completing the proof of
Lemma~\ref{lem:leave-two}.

\section{Omitted proofs in the proof of Proposition~\ref{prop:bridge}}
\label{app:omitted-bridge}

This appendix completes the remainder estimate used in
Subsection~\ref{subsec:proof-bridge}, reusing the shared bounds proved in
Appendix~\ref{app:omitted-collapsed}.

\subsection{Proof of Lemma~\ref{lem:bridge-remainders}}
\label{subsec:proof-bridge-remainders}

Lemma~\ref{lem:orphan}, identity~\eqref{eq:occupancy-hit-probability},
and $\rho_{\tarstt}\leq3\effhorizon^2$ give
\[
 \norm{\widehat\Delta_\subG^{\rm orp}(\tarstt)}_{L^1}\leq3\effhorizon^2e^{-n\nuMin/\effhorizon}.
\]
On $\mathcal A_\subG$, $\widehat\Delta_\subG^{\rm orp}(\tarstt)=\widehat\Delta_{\rm TD}(\tarstt)-\widehat\Delta_\subG^{\rm RB}(\tarstt)$.  Lemma~\ref{lem:tail} and conditional
Jensen give
\[
 \norm{\widehat\Delta_\subG^{\rm orp}(\tarstt)\one_{\mathcal A_\subG}}_{L^4}
 \leq \norm{\widehat\Delta_{\rm TD}(\tarstt)}_{L^4}
      +\norm{\widehat\Delta_\subG^{\rm RB}(\tarstt)\one_{\mathcal A_\subG}}_{L^4}
 \leq Cn^2\effhorizon^3.
\]
Here condition~\eqref{eq:oracle-sample-size} implies $n\nuMin\geq1$, so
inequality~\eqref{eq:collapsed-fourth-moment} is absorbed by the last
quantity.  For any random variable $X$ with finite fourth moment,
$\|X\|_{L^2}^2\leq\|X\|_{L^1}^{2/3}\|X\|_{L^4}^{4/3}$ by H\"older's
inequality.  Applying this with
$X=\widehat\Delta_\subG^{\rm orp}(\tarstt)$ and
$\kappa_\subG=n\nuMin/\effhorizon$ therefore yields
\begin{equation}
 \E[\widehat\Delta_\subG^{\rm orp}(\tarstt)^2\one_{\mathcal A_\subG}]
 \leq Cn^{8/3}\effhorizon^{16/3}e^{-2\kappa_\subG/3}
 \leq C\mathfrak R_\subG.
 \label{eq:orphan-interpolation-bound}
\end{equation}

Likewise, by identity~\eqref{eq:occupancy-hit-probability} and
inequality~\eqref{eq:covered-subgraph-cardinality},
\[
 \Pp(\mathcal A_\subG^c)
 \leq |\subG|e^{-n\nuMin/\effhorizon}
 \leq(\effhorizon/\nuMin)e^{-n\nuMin/\effhorizon}.
\]
Lemma~\ref{lem:tail} and H\"older's inequality show that
\begin{equation}
 \E[\widehat\Delta_{\rm TD}(\tarstt)^2\one_{\mathcal A_\subG^c}]
 \leq Cn^4\effhorizon^6(\effhorizon/\nuMin)^{1/2}e^{-\kappa_\subG/2}
 \leq C\mathfrak R_\subG.
 \label{eq:coverage-holder-bound}
\end{equation}
For the second inequalities in \eqref{eq:orphan-interpolation-bound}
and~\eqref{eq:coverage-holder-bound}, note that the two branches of
$\mathfrak R_\subG$ equal $\effhorizon^6\ell_\subG^2/\kappa_\subG^2$ and
$\effhorizon^5\ell_\subG^5/\kappa_\subG^2$.  The assumption gives
$\kappa_\subG\geq C\effhorizon^2\ell_\subG^4$.  Since
$e^{\ell_\subG}=en\effhorizon /\nuMin=e \kappa_\subG(\effhorizon/\nuMin)^2$, the ratios of the two exponential
middle terms to the first branch are bounded, respectively, by
\[
 C \kappa_\subG^{10/3}\ell_\subG^{-2}
   \exp\{4\ell_\subG/3-2\kappa_\subG/3\},
 \qquad
 C \kappa_\subG^{15/4}\ell_\subG^{-2}
   \exp\{9\ell_\subG/4-\kappa_\subG/2\}.
\]
Using also $e^{\ell_\subG}\kappa_\subG=en^2$, the corresponding ratios to the second
branch are bounded by
\[
 C \effhorizon^{1/3}\kappa_\subG^{10/3}\ell_\subG^{-5}
   \exp\{4\ell_\subG/3-2\kappa_\subG/3\},
 \qquad
 C \effhorizon \kappa_\subG^{15/4}\ell_\subG^{-5}
   \exp\{9\ell_\subG/4-\kappa_\subG/2\}.
\]
All four ratios are bounded by a universal constant.  Indeed,
$\effhorizon,\ell_\subG\geq1$ and the sample-size condition imply
$\kappa_\subG\geq9\ell_\subG$ after enlarging the universal constant.  Hence
\[
 \frac43\ell_\subG-\frac23\kappa_\subG\leq-\frac13\kappa_\subG,
 \qquad
 \frac94\ell_\subG-\frac12\kappa_\subG\leq-\frac14\kappa_\subG.
\]
Also, $\kappa_\subG\geq C\effhorizon^2\ell_\subG^4$ bounds
$\effhorizon^{1/3}$ by $C\kappa_\subG^{1/6}\ell_\subG^{-2/3}$ and
$\effhorizon$ by $C\kappa_\subG^{1/2}\ell_\subG^{-2}$.  Since
$\ell_\subG\geq1$, each ratio is bounded by a fixed power of
$\kappa_\subG$ times either $e^{-\kappa_\subG/3}$ or
$e^{-\kappa_\subG/4}$, which has a finite universal supremum.
This proves the asserted absorptions into both branches and hence into
$\mathfrak R_\subG$.
The correction is zero on $\mathcal A_\subG^c$, so the left side of
\eqref{eq:orphan-interpolation-bound} is its unconditional second moment.
This bound and~\eqref{eq:coverage-holder-bound} prove
Lemma~\ref{lem:bridge-remainders}.

\section{Proof of Miscellaneous Results}
\label{app:miscellaneous-proofs}

This appendix proves the Monte Carlo risk bound~\eqref{eq:mc-finite-risk}
and the target-state guarantee in Proposition~\ref{prop:td-mc}, followed by
the subgraph variance bound~\eqref{eq:subgraph-exit-upper}.

\subsection{Proof of equation~(\ref{eq:mc-finite-risk})}
\label{subsec:proof-mc-risk}
Conditional on which trajectories hit $\tarstt$, their returns following
the first hit are independent and have the return distribution from
$\tarstt$, by the strong Markov property. Thus their average has conditional
mean $\Vstar(\tarstt)$ and conditional variance
$\Var_{\tarstt}(\sum_{t<T}R_t)/H_{\tarstt}$ when $H_{\tarstt}>0$.
When $H_{\tarstt}=0$, the squared error is $\Vstar(\tarstt)^2$.
This proves the equality in~\eqref{eq:mc-finite-risk}.
For $H\sim\operatorname{Binom}(n,p)$,
\[
 \E\frac{\one\{H>0\}}H
 \leq 2\E\frac1{H+1}
 =\frac{2[1-(1-p)^{n+1}]}{(n+1)p}
 \leq\frac2{np}.
\]
The bounds $|\Vstar|\leq h$ and $(1-p)^n\leq e^{-np}$ now give the
inequality in~\eqref{eq:mc-finite-risk}, completing the calculation.

\subsection{Proof of Proposition~\ref{prop:td-mc}}
\label{subsec:proof-td-mc}
Take $\subG=\{\tarstt\}$. In the collapsed sample, let $C$ count transitions
from $\tarstt$ back to $\tarstt$, and let $Y$ be the sum of collapsed rewards.
Every trajectory that hits $\tarstt$ has exactly one final collapsed
transition to $\termState$, so $N_{\tarstt}-C=H_{\tarstt}$.
Moreover, the collapsed rewards along each such trajectory telescope to
the return following its first visit. On $\{H_{\tarstt}>0\}$,
\[
 \widehat V_{\{\tarstt\}}(\tarstt)
 =\frac{Y/N_{\tarstt}}{1-C/N_{\tarstt}}
 =\frac Y{H_{\tarstt}}
 =\widehat V_{\rm MC}(\tarstt).
\]
Both estimators equal zero when $H_{\tarstt}=0$.

The variance identity can also be checked directly. Put
$f=F(\tarstt,\tarstt)$ and
$v=\Var_{\tarstt}(\sum_{t<T}R_t)$.
Successive excursions from $\tarstt$ give
\[
 F_{\{\tarstt\}}(\tarstt,\tarstt)=f,\qquad
 \nu(\tarstt)=p_{\tarstt}f,\qquad
 v=f\sigma_{\{\tarstt\},\Vstar}^2(\tarstt).
\]
The last equality follows by summing the conditional variances of the
orthogonal stopped martingale increments; effective-horizon moments justify
stopping and passage to the limit. Consequently,
\[
 \sigma_{\{\tarstt\}}^2(\tarstt)
 =\frac{f^2\sigma_{\{\tarstt\},\Vstar}^2(\tarstt)}{p_{\tarstt}f}
 =\frac v{p_{\tarstt}}=\sigma_{\rm MC}^2(\tarstt).
\]
For this subgraph, $\nuMin=\nu_0$ and $\ell_\subG=\ell_0$.
Theorem~\ref{thm:oracle} therefore gives~\eqref{eq:td-mc-risk}, completing
the proof of Proposition~\ref{prop:td-mc}.

\subsection{Proof of inequality~(\ref{eq:subgraph-exit-upper})}
\label{subsec:proof-subgraph-exit-upper}
Fix $s\in\subG$. Orthogonality of the stopped Bellman residuals gives
\[
 \sigma_{\subG,\Vstar}^2(s)
 =\sigma_{\Vstar}^2(s)
  +\E_s\sum_{1\leq t<T_\subG}\sigma_{\Vstar}^2(S_t).
\]
The bounds $T_\subG\leq T$, $|\delta_t|\leq3\effhorizon$, and
$\E_sT^2<\infty$ justify stopping and passage to the limit in $L^2$.
For every nonterminal state $x$, the same orthogonality and the telescoping
identity for the complete return yield
\[
 \E_x\sum_{t<T}\sigma_{\Vstar}^2(S_t)
 =\Var_x\!\left(\sum_{t<T}R_t\right)
 \leq\E_xT^2\leq4\effhorizon^2.
\]
If $S_1\in\subG\cup\{\termState\}$, the sum after the first step is empty.
Otherwise, the Markov property and nonnegativity of the residual variances
give
\[
 \E_s\sum_{1\leq t<T_\subG}\sigma_{\Vstar}^2(S_t)
 \leq\sum_{x\in\cS\setminus\subG}P(s,x)
       \E_x\sum_{t<T}\sigma_{\Vstar}^2(S_t)
 \leq4\effhorizon^2q_\subG^{\rm exit}(s).
\]
Collapsing preserves every visit to $\subG$, so
$F_\subG(\tarstt,s)=F(\tarstt,s)$ for $s\in\subG$.
Substituting these bounds into~\eqref{eq:collapsed-leading-variance} and
bounding the primitive-variance sum over $\subG$ by the full sum
in~\eqref{eq:td-variance} proves~\eqref{eq:subgraph-exit-upper}.

\section{Simulation setup and implementation details}
\label{app:simulation-details}

\paragraph{Model construction and replication.}
At each size, the scale experiment uses seeds $20261004,\ldots,20261008$.
The four-structure experiment uses the first three seeds for
each random-graph family and for the lattice rewards. The motivating
graph uses seed $20261004$. Each MRP, including its reward means, is
held fixed across trajectory datasets. Graph families, sizes, and time
scales were informed by exploratory calculations; the reported design
was fixed before the reported Monte Carlo runs. The comparisons use
all instances specified by that design.

The lattice has ordinary nearest-neighbor edges and no wraparound; at a
boundary state the walk chooses uniformly among the available neighbors.
In the queue, the three nonterminal event probabilities are
$0.3$, $0.35$, and $0.35$, multiplied by the survival probability $0.99$.
An arrival at a full first buffer leaves the state unchanged.
A station-1 service transfers one job to station 2 only when station 1
is nonempty and station 2 is not full. A station-2 service removes one
job only when station 2 is nonempty. An unavailable service leaves the
state unchanged. Termination occurs with probability $0.01$.
All rewards lie in $[0,1]$, and trajectory lengths have geometric tails.

The replication package contains the model constructors, seeds,
per-dataset estimates, and population benchmarks. For experiment-specific
instance index $i$, sample size $n$, and replication $r$, the trajectory
generator uses the seed sequence $(671920261004,i,n,r)$.
The reported errors average over complete trajectory datasets conditional
on each fixed MRP. They include the zero estimate on datasets that miss
the target. In particular, filtering displayed points by expected target
hits does not condition the errors on observing the target.

\paragraph{Estimator implementation and comparison sets.}
TD includes terminal transitions in its empirical row denominators, estimating
the termination probabilities along with the other transitions. The subgraph
estimator collapses trajectories at successive visits to $G_n^{(2)}$, including
excursions that leave and re-enter the set. The comparison sets use population
occupancy and hence are deterministic, rather than selected from the data.

The cutoff $c=2$ is motivated by eliminating a once-visited state from the
empirical Bellman system: its single-successor row replaces the corresponding
steps by their accumulated reward. Keeping the target and the repeatedly
visited states gives a random set, distinct from $G_n^{(2)}$.
The deterministic set need not satisfy the sufficient coverage condition in
Theorem~\ref{thm:oracle}.

\paragraph{Averaging and confidence bands.}
At each common sample-size doubling index $k$, we average both coordinates
over the $m$ instances in a panel, normalizing each error by that instance's
TD variance first. We retain only stages with at least ten expected target
hits in every instance. No interpolation is used. If $a_i$ and $s_i$ are
an instance's plotted estimate and Monte Carlo standard error, the mean
$\bar a=m^{-1}\sum_i a_i$ has approximate pointwise $95\%$ confidence
interval $\bar a\pm1.96\sqrt{\sum_i s_i^2}/m$.
These intervals condition on the fixed MRPs and quantify trajectory-sampling
uncertainty, not variation across a population of graphs. Population variance
benchmarks are averaged without confidence bands. For subgraph comparisons,
$a_i$ is the within-instance MSE ratio and $s_i$ its paired delta-method
standard error, accounting for the shared datasets.

\paragraph{Population benchmarks and numerical checks.}
Write $F=(I-P)^{-1}$ and
$q(s)=\operatorname{Var}_s(R_0+\Vstar(S_1))$.
For $B=\mathcal S\setminus G$, the stopped residual variance on $G$ is
\[
 q_G=q|_G+P_{GB}(I-P_{BB})^{-1}q|_B,
\]
and the fundamental matrix of the collapsed chain is $F|_{G\times G}$.
Substituting these quantities into
\eqref{eq:collapsed-leading-variance} gives the variance curves.
The fraction of target future visits retained by $G$ is
$\sum_{s\in G}F(s_0,s)/\sum_{s\in\mathcal S}F(s_0,s)$.
Population linear systems are solved numerically with residual checks.
For numerical stability, occupancies below $10^{-13}$ times the
largest occupancy are treated as zero; such states are below the
comparison thresholds at every reported sample size.

Numerical checks recover the MC and TD variance benchmarks and the
first-visit MC and TD estimators at the singleton and full-support endpoints,
respectively. Eliminating once-visited states also preserves the directly
solved target estimate. These are numerical identities; the deterministic
population threshold remains a separate comparison.

\end{document}